\documentclass[10pt,journal,compsoc]{IEEEtran}
\usepackage{natbib}
\setcitestyle{square,sort,comma,numbers}
\usepackage{multirow, booktabs}
\usepackage{siunitx}
\usepackage{makecell}
\usepackage{longtable}
\usepackage{mathrsfs}
\usepackage{graphicx}
\usepackage[font=small,labelfont=bf]{caption}
\usepackage{amsthm}

\usepackage{algorithm}
\usepackage{algcompatible}
\usepackage{epstopdf} 
\usepackage{subcaption}
\usepackage{lipsum}
\usepackage{amssymb,amsmath,amsthm}

\newtheorem{lemma}{Lemma}

\newtheorem{problem}{Problem}
\newtheorem{definition}{Definition}
\usepackage{mathtools} 
\usepackage{hyperref}
\hypersetup{
	colorlinks=true,
	linkcolor=blue,
	filecolor=blue,      
	urlcolor=blue,
	citecolor=blue
}
\renewcommand{\COMMENT}[2][.2\linewidth]{%
	\leavevmode\hfill\makebox[#1][l]{//~#2}}

\begin{document}
	\newcommand*{\everymodeprime}{\ensuremath{\prime}}
	\title{Transfer-based adaptive tree for multimodal sentiment analysis based on user latent aspects}
	\author{Sana~Rahmani$^1$, Saeid~Hosseini$^{2*}$, Raziyeh~Zall$^{1*}$, Mohammad~Reza~Kangavari$^1$, Sara~Kamran$^1$, Wen~Hua$^3$ \IEEEcompsocitemizethanks{
			\IEEEcompsocthanksitem$^1$ S. Rahmani, R. Zall, M. Kangavari, S. Kamran are with School of computer engineering, Iran University of Science and Technology, Iran. \protect\\Email: {sana\_rahmani}@comp.iust.ac.ir, {zall\_razieh}@iust.ac.ir, {kanagvari}@iust.ac.ir, and {sara.kamran72}@gmail.com
			\IEEEcompsocthanksitem$^2$ S. Hosseini is with the Faculty of Computing and Information Technology, Sohar University, Sohar, Oman. Email: {sahosseini}@su.edu.om
			\IEEEcompsocthanksitem$^3$ W. Hua is with the School of Information Technology and Electrical Engineering, University of Queensland, Australia. E-mail: {w.hua}@uq.edu.au
			\IEEEcompsocthanksitem$^*$ S. Hosseini and R. Zall contributed equally
		}
		\thanks{}}
	\markboth{}
	{Shell \MakeLowercase{\textit{et al.}}: Bare Demo of IEEEtran.cls for Computer Society Journals}
	\IEEEtitleabstractindextext{
		\begin{abstract}
		Multimodal sentiment analysis benefits various applications such as human-computer interaction and recommendation systems. It aims to infer the users' bipolar ideas using visual, textual, and acoustic signals. Although researchers affirm the association between cognitive cues and emotional manifestations, most of the current multimodal approaches in sentiment analysis disregard user-specific aspects. To tackle this issue,  we devise a novel method to perform multimodal sentiment prediction using cognitive cues, such as personality. Our framework constructs an adaptive tree by hierarchically dividing users and trains the LSTM-based submodels, utilizing an attention-based fusion to transfer cognitive-oriented knowledge within the tree. Subsequently, the framework consumes the conclusive agglomerative knowledge from the adaptive tree to predict final sentiments. We also devise a dynamic dropout method to facilitate data sharing between neighboring nodes, reducing data sparsity. The empirical results on real-world datasets determine that our proposed model for sentiment prediction can surpass trending rivals. Moreover, compared to other ensemble approaches, the proposed transfer-based algorithm can better utilize the latent cognitive cues and foster the prediction outcomes. Based on the given extrinsic and intrinsic analysis results, we note that compared to other theoretical-based techniques, the proposed hierarchical clustering approach can better group the users within the adaptive tree.	
		\end{abstract}
		\vspace{-3mm}
		}		
		\maketitle
		\IEEEpeerreviewmaketitle
		\IEEEraisesectionheading{\section{Introduction}\label{introductionSection}}
		\IEEEPARstart{A}{s} a branch of affective computing, sentiment analysis can classify user-generated content as positive or negative \cite{Poria2017}. This research area gains growing interest in various domains: (i) In Human-Computer Interaction, to expand the interchange between the agents and individuals \cite{PerezGaspar2016}. (ii) In political forecasting, to study political sentiments \cite{Ebrahimi2017}. (iii) In recommendation systems, to analyze the reviews of consumers \cite{Li2016}. (IV) In market forecasting, to promote financial prediction outcome \cite{Ren2018}. With the rapid rise of social media, individuals mostly tend to convey their opinions in videos, creating the opportunity to analyze multiple input channels known as Acoustic, Visual, and Textual modals. Such multimodal input carries surplus information on the individual's latent cues, including personality, motivations, and other ignored parameters in current sentiment analysis methods. Given a video clip $v_i$ for user $l_k$, we aim to employ a cognitive-based ensemble approach to improve sentiment analysis for $v_i$ through considering the cognitive cues of the user $l_k$, denoted by $C_k$. However, certain \textit{challenges} abound:\\
		\textit{\textbf{Challenge 1 (Missing Cognitive Annotations)}}\\
		\indent The sentiment analysis task needs the annotations to represent the cognitive cues of the user in the input video. However, the majority of available datasets on multimodal affective interpretation, including MOSI \cite{Zadeh2016}, MOUD \cite{PerezRosas2013}, and IEMOCAP \cite{Busso2008}, are not enriched by cognitive annotations. Aiming to augment the datasets using cognitive cues, one can leverage the self and full-supervised methods. Additionally, utilizing the pre-trained models may seem partially inaccurate because they include the perturbations from the initial trainee dataset. To address this issue, one may opt for devising a framework to reduce the cross dataset effects or choose a labeled dataset with the minimum track of skews.\\
		\textit{\textbf{Challenge 2 (Assorted Categories Affected by Various Cues)}}\\
		\indent The resultant categories constructed by the cognitive-oriented subset-based ensembles can have substantial similarities or overlook information-theoretical parameters including, entropy. Moreover, the unbiased categories must adapt to the deep learning approaches to avoid starving or overfitting. In retrospect, choosing the correct values for adjustment parameters in categorization can be challenging. Dimension, cut-point, and the number of clusters or hierarchical levels are samples of such parameters.
		\begin{figure}[H]
			\centering
			\vspace{-4mm}
			\begin{subfigure}{0.49\linewidth}
				\centering
				\includegraphics[width=1\textwidth]{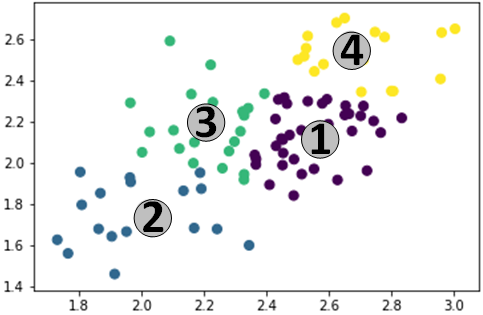}
				\vspace{-5mm}
				\caption{Data categorization on MOSI\cite{Zadeh2016}}
				\label{fig:clustervis}
			\end{subfigure}
			\centering
			\begin{subfigure}{0.49\linewidth}
				\centering
				\includegraphics[width=1\textwidth]{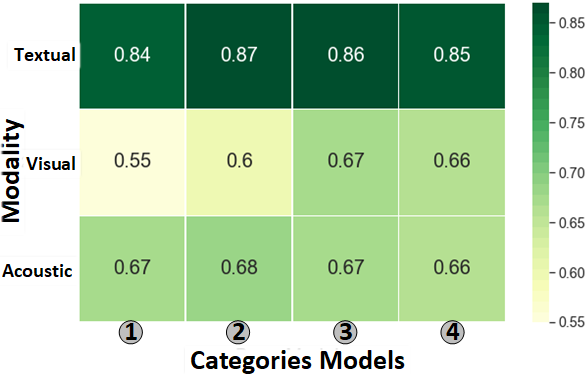}
				\vspace{-5mm}
				\caption{modals accuracy comparison}
				\label{fig:Heat-categoriez}
			\end{subfigure}
			\vspace{-3mm}
			\caption{Base Models Accuracy Comparison}
			\vspace{-4mm}
			\label{fig:Heatmap2}
		\end{figure}		
		\par\noindent\textbf{Observation} To demonstrate how cognitive cues can partition a set of data instances, we set up an observation that performs clustering on the videos in the MOSI affective dataset. As depicted in Fig. \ref{fig:clustervis}, we utilize the Euclidean distance to classify the videos into four groups. The data points and individual colors respectively signify videos and pertinent clusters. Furthermore, to investigate the effect of categorization on emotional manifestation, Fig. \ref{fig:Heat-categoriez} shows the accuracy heatmap to compare the modalities of constructed clusters.
		The performance of the sentiment analytics on video transcripts is higher than other input channels, reaching up to \%87 on accuracy. Moreover, the accuracy deviation in various clusters of the visual channel implies that the level of cognition can directly affect the sentiment understanding. Hence, besides the effectiveness of the categorization algorithms, the availability of sufficient cognitive cues can foster the sentiment analysis procedure.\\
		\textit{\textbf{Challenge 3 (Data Scarcity in Subset Categories)}}\\
		\indent Intuitively, we perform top-bottom hierarchical categorization on the adaptive tree to better exploit the cognitive correlation between subsets. However, dividing the data in each node into portions to form subsets can cause data scarcity, where the data enclosed with some nodes may not suffice to train the model, affecting the learning capability of the deep architecture. Relying merely on the pruning procedure to control the expansion of the adaptive tree based on the correlation criteria can negatively affect the performance of the cognitive-based categorization. To address this issue, we urge that besides the pruning supervision, the subsets must also be enriched to preserve the original subset characteristics and avoid unintended influence on models. \vspace{1mm}\\
		\noindent\textbf{Contributions.} While most of the current works \cite{Zadeh2017,Poria2017a,Majumder2018,Huddar2018}
		neglect the indispensable role of the cognitive cues in multimodal sentiment analysis, our proposed framework in this paper enhances analytical results by leveraging user-specific latent characteristics. To the best of our knowledge, we reveal the utopian spirit of the hidden parameters and devise the first approach based on an adaptive tree that utilizes an attention-based fusion to facilitate cognitive-oriented knowledge transfer within the tree, forming an ensemble of submodels.
		In retrospect, where our previous works \cite{Najafipour2020,Hosseini2020,Hua2016} focus on short-text understanding, we propose an approach that not only takes advantage of the textual analytics but also infers the cognitive cues from acoustic and visual mediums to leverage the contrast between individuals and effectively predict the sentiments. Our contributions are fourfold:\\
		\vspace{-6mm}
		\begin{itemize}
			\item We develop a novel transfer-based adaptive tree to facilitate ensembles on multimodal submodels.
			\item We analyze multiple categorization strategies to classify individuals by leveraging pertinent latent cues.
			\item We design a dynamic dropout approach to regularize data of relevant classes during hierarchical training. 	
			\item We propose a unified cognitive-based framework for sentiment analysis on multimodal datasets.
			\end{itemize}
		\vspace{-2mm}
		\par\noindent The rest of the paper is as follows: in Sec. \ref{relatedwork}, we study the literature; in Sec. \ref{problemstatement}, we provide the problem and our framework; in Sec. \ref{methodology} and \ref{experiments} we explain our model and experiments. Lastly, we conclude the paper in Sec. \ref{conclusion}\vspace{-6mm}.	
		\section{Related Work}
		\label{relatedwork}
		\vspace{-1mm}
		As Table \ref{tab:literaturereview} briefs, the literature is threefold: multimodal sentiment analysis, transfer learning, and leveraging individual characteristics\vspace{-5mm}.
		\subsection{Multimodal Sentiment Analysis}
		\label{MultiModal_Sentiment_Analysis}
		\vspace{-1mm}
		Sentiment analysis as a branch of affective computing aims to exploit the polarity expressed by an individual in three possible categories of positive, negative, or neutral \cite{Poria2017}. Accordingly, numerous research works analyze the sentiments from each of the modals (unimodal), including text \cite{Song2020}, audio \cite{Sun2017}, and video \cite{Kessous2010}. Moreover, multimodal affective analysis has adapted various input modals \cite{Corneanu2016}, including acoustic-textual \cite{Eyben2010}, acoustic-visual \cite{Kessous2010,Woellmer2013}, and the tri-modal unified frameworks \cite{Zadeh2017,Poria2017a}. The main challenge in the multimodal analysis is the fusion method \cite{Poria2017} that comprises two literature domains of decision and feature level, facilitating a variety of modals combinations. Weighted voting \cite{Evangelopoulos2013} and Bayes belief integration \cite{Chanel2011} are two instances of decision-level fusion, fusing independent predictions of each modal into a unified output. The feature-level approaches such as Tensor \cite{Zadeh2017} and attention-based \cite{Poria2017a} generate the modal-specific embeddings to feed into another final analysis module. Nevertheless, various architectures adapt models, including Multiple Kernel Learning (MKL) on Support Vector Machine (SVM) classifiers \cite{Poria2017b} and Recurrent Neural Networks (RNN) classes of Gated-Recurrent Unit (GRU)\cite{Majumder2018} and Long Short-Term Memory (LSTM) \cite{Woellmer2013,Poria2017a}, to target the context-dependent analysis of inputs. Ensemble Learning approaches utilize various modules to improve the performance by combining the results \cite{Zall2016,Zall2019}. Some train the same model on distinct datasets, as \cite{Rozgic2012} utilizes various SVM classifiers trained on different dataset categories distinguished by emotion labels. Others train divergent models on the same dataset, as \cite{Huddar2018} applies SVMs together with logistic regression, decision tree, random forest, and clustering modules on the same corpus. In retrospect, we propose an ensemble on multimodal LSTM models equipped with an attention-based fusion trained on data divisions based on the composer’s latent cues\vspace{-4mm}.
		\subsection{Transfer Learning}
		\vspace{-1mm}
		Transfer Learning (TL) approaches benefit from pre-trained models to improve learning by transferring knowledge between domains or extracting new representations of input data as feature sets \cite{Liu2019}. TL module adopts into unimodal analysis approaches of textual, visual, and acoustic. For textual, \cite{Felbo2017} utilizes TL to perform textual analysis on emojis, and \cite{Hazarika2021} applies transformer encoding to detect emotions in conversations, passing multi-turn conversation parameters to an emotion classifier. For visual, \cite{TamilPriya2020} applies TL on large-scale image classification to analyze visual cues in affective analysis. Finally, \cite{Deng2017} applies TL in sparse auto-encoding to accomplish acoustic emotion prediction.\\
		Furthermore, multimodal affective analysis adopts TL in various applications: For instance, \cite{Kaya2017} detects audio-visual emotions by obtaining CNN features of facial expressions through the TL module. In retrospect, aiming to reduce the effect of catastrophic forgetting in TL, the multi-task learner in \cite{Gideon2017} adapts into multimodal affective analysis by learning progressive networks on paralinguistic tasks, recognizing emotions, speakers, and gender. Similarly, \cite{Akhtar2019} forms the multi-task learner by appending a classifier to each emotion, where the work extends to \cite{Akhtar2020} by integrating topical classification. Accordingly, we perform multimodal analysis with a novel ensemble through a hierarchical knowledge transfer in an adaptive tree based on users' latent characteristics\vspace{-5mm}.
		\subsection{Individual Characteristics}
		\vspace{-1mm}
		In this section, we review the personal characteristics that can influence users’ expressions. We get inspired by cognitive studies that affirm how significantly personality factors can affect human behavior \cite{Verduyn2012}. Accordingly, we survey the literature assuming five personality traits as the primary latent cues of humans. The research on uni-modal personality analysis comprises trilateral input modals of acoustic \cite{Mohammadi2012}, textual \cite{Majumder2017}, and visual \cite{TeijeiroMosquera2014}. Likewise, there is comprehensive research to investigate multimodal personalities, \cite{Guecluetuerk2017} leverages a deep residual network that can predict personality on the first impression for audio-visual input.\\
		Similarly, Batrinca et al. \cite{Batrinca2016} apply audio-visual predictions on a Human-Computer Interaction (HCI) task, and \cite{Aslan2021} collectively uses quadrilateral modals of audio, text, face, and video background. Other works \cite{Basak2018}\cite{Durupinar2009} examine personality in the context of synthetic applications, simulating the crowd through constructing artificial agents designated by personality weights that replicate emotion contagion in incidents. From another perspective, \cite{Yang2019,Capuano2015,Lin2017} exploit applications on personality perception. While \cite{Yang2019} devises a user-centric recommendation system where the personality factor determines the similarity weights between users and items, \cite{Capuano2015} leverages personality scores from social networks to improve the user interaction with information systems. Comparably, \cite{Lin2017} utilizes the user’s personality to improve the textual sentiment evaluation in short texts. Nevertheless, the majority of practices in affective evaluations that consider user characteristics are solely unimodal. In response, we utilize user-specific characteristics in a multimodal analytical system to improve sentiment prediction results.
		\begin{table}[H]
			\vspace{-2mm}
			\centering
			\caption{Literature Review}
			\vspace{-3mm}
			\def\arraystretch{1.5}
			\begin{tabular}{cll}
				\Xhline{2\arrayrulewidth}
				Category & Approaches& References                                                                                                                                                                                                                                                  \\ \hline
				\multirow{2}{*}{\begin{tabular}[c]{@{}c@{}}Multimodal\\ Sentiment Analysis\end{tabular}}
				& Traditional Learning&\cite{Rozgic2012}\cite{Huddar2018}\\
				& Deep Learning &\cite{Woellmer2013}\cite{Zadeh2017}\cite{Poria2017a}\cite{Poria2017b}\\ \hline
				\multirow{2}{*}{\begin{tabular}[c]{@{}c@{}}Transfer Learning\end{tabular}} 
				& Knowledge Transfer & \cite{Hazarika2021}\cite{Gideon2017}\cite{Akhtar2019}\cite{Akhtar2020}\\
				& Feature Extraction & \cite{Felbo2017}\cite{TamilPriya2020}\cite{Deng2017}\cite{Kaya2017}\\ \hline
				\multirow{2}{*}{\begin{tabular}[c]{@{}c@{}}Individual\\ Characteristics\end{tabular}} 
				& Prediction & \cite{Mohammadi2012}\cite{Majumder2017}\cite{TeijeiroMosquera2014}\cite{Guecluetuerk2017}\cite{Batrinca2016}\cite{Aslan2021}\\
				& Perception & \cite{Basak2018}\cite{Durupinar2009}\cite{Yang2019}\cite{Capuano2015}\cite{Lin2017}\\
				\Xhline{2\arrayrulewidth}
			\end{tabular}
			\label{tab:literaturereview}
			\vspace{-5mm}
		\end{table}
		\section{Problem Statement}
		\label{problemstatement}		
		\vspace{-1mm}
		In this section, we elucidate primary definitions, the problem statements, and the proposed framework\vspace{-5mm}.
		\subsection{Preliminary Concepts}
		\label{sec:preliminary concepts}
		\vspace{-1mm}
		\begin{definition} (Sentiment Value)
			\label{def:Sentiment vlaue}			
			The Sentiment value denoted by $s \in \{0,1\}$ is a binary value, representing a positive or negative opinion regarding an entity\vspace{-2mm}.
		\end{definition}
		\begin{definition} (Video)
			\label{def:video}			
			A video denoted by $v_i \in \mathbb{V}$ reflects the participating user, $l_k$ reactions and is associated with a sentiment value, $v_i.s$, start time, $v_i.t_s$, end time, $v_i.t_f$, and tri-modal data, including, Acoustics ($v^a_i$), Visuals ($v^v_i$), and Textual ($v^t_i$) contents\vspace{-2mm}.
		\end{definition}
		\begin{definition} (Utterance)
			\label{def:utterance}			
			Utterance, $u_{j,i} \in \mathbb{U}$, represents the partition $j$ of video $v_i$, bounded by pauses. Each utterance $u_{j,i}$ contains a sentiment value $u_{j,i}.s$, start time $u_{j,i}.t_s$, and the end time $u_{j,i}.t_f$, such that:
			\vspace{-2mm}
			\begin{itemize}
				\item The bounding rule applies as $[u_{j,i}.t_s , u_{j,i}.t_f] \subseteq [v_i.t_s , v_i.t_f]$.
				\item Utterances shall not overlap with each other in a video, $\forall m , n$, if $m \neq n$ then $[u_{m,i}.t_s , u_{m,i}.t_f] \cap [u_{n,i}.t_s , u_{n,i}.t_f] = \O$.
				\vspace{-1mm}
			\end{itemize}
			Furthermore, any $u_{j,i}$ is associated with Acoustics ($u^a_{j,i}$), Visuals ($u^v_{j,i}$), and Textual contents ($u^t_{j,i}$). Also, each video $v_i$ is composed of a set of utterances, $v_i=\{u_{j,i}|1\leq j\leq \digamma_v^{len}(v_i)\}$, where $\digamma_v^{len}: \mathbb{V}\rightarrow\mathbb{Z}$ is the function retrieving the number of utterances in the given video and $v_i.s$ represents the majority aggregation on sentiment values of the pertinent utterances\vspace{-2mm}. 
		\end{definition}
		\begin{definition} (User)
			\label{def:user}
			Each video $v_i$ is associated with a user $l_k \in \mathbb{L}$ that expresses a sentiment value for every contained utterance $u_{j,i}$. Given one or more videos composed by user $l_k$, we can utilize the function $\digamma_l^v:\mathbb{L}\rightarrow\mathbb{V}$ to retrieve pertinent videos\vspace{-2mm}.
		\end{definition}	
		\begin{definition} (Cognitive Cue)
			Each cognitive cue, denoted by $c^y \in \mathbb{R}$, can constitute any latent cue, like personality, that can distinguish individuals.
			\vspace{-5mm}
		\end{definition}
		\subsection{Problem Definition}		
		\vspace{-1mm}		
		\begin{problem} (multimodal sentiment analysis)
			\label{prob:cognitive-based}
			Given a video $v_i$ of the user $l_k$, our goal is to improve the estimation of sentiments through utilizing the latent cues of $l_k$\vspace{-2mm}.
		\end{problem}	
		\begin{problem} (enhancing intra-ensemble awareness)
			\label{prob:ensemble-tree}
			Given a set of submodels $M$, where each inference submodel $m_i \in M$ partially contributes to the prediction process, our goal is to enrich the knowledge submission between submodels in the ensemble\vspace{-2mm}.
		\end{problem}	
		\begin{problem} (reducing data scarcity)
			\label{prob:borrow-neighbour}
			Given a set of videos $V$, our goal is to enrich any subset $\tilde{V} \subseteq V$ with data insufficiency\vspace{-5mm}.
		\end{problem}	
		\subsection{Framework Overview}
		\vspace{-1mm}
		Fig. \ref{fig:Framework} illustrates our proposed unified framework that employs cognitive cues to improve the sentiment prediction. Through the offline phase, we firstly feed tri-modal visual, textual, and acoustic data into feature extraction modules to construct an aligned, normalized set for each of the Modals. For data augmentation, we develop an inference algorithm that extracts latent aspects of individuals. To follow, we devise a novel adaptive tree with each node comprising a user set, pertinent collective knowledge core, and a transfer-based submodel. On the one hand, we employ clustering and theoretical-based approaches in a top-bottom procedure to perform hierarchical categorization of users based on their cognitive cues, and on the other hand, we train the transfer-based submodel and form the knowledge core in parallel through a bottom-up process. Accordingly, the LSTM submodel consumes the attention-based fusion that comprises node-specific user data and the knowledge core, aggregated by the submodels of the child nodes. Moreover, a nontrivial module adopts users' data from similar neighbors to tackle data sparsity in training submodels within immature nodes through a dynamic dropout approach. Finally, the prediction module utilizes the adaptive tree to evaluate the sentiments based on user-specific latent cues\vspace{-4mm}.
		\begin{figure}[htp]
			\centering
			\includegraphics[width=1\linewidth]{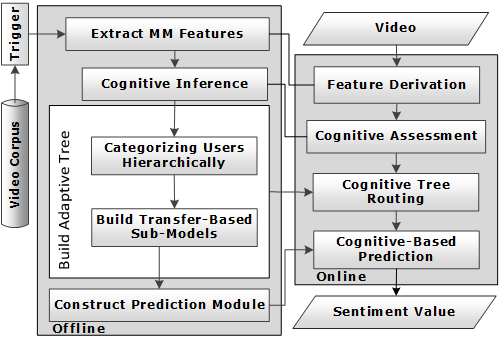}
			\vspace{-7mm}
			\caption{Framework}
			\label{fig:Framework}
			\vspace{-4mm}
		\end{figure}
	
		\noindent In the online phase, we pass each input video into feature extraction and cognitive assessment modules adopted from the offline section. Given the cognitive cues of the video composer, we then find a top-bottom route that leads to the most relevant leaf node. Conversely, we take the route to transfer the knowledge by executing the submodels in a bottom-up manner. Finally, given the collective knowledge of the tree, the prediction module performs the final cognitive-based sentiment prediction of the input data\vspace{-5mm}.
		\section{Methodology}
		\label{methodology}
		\vspace{-1mm}
		\subsection{Offline Phase} 
		\label{offline_phase}
		\vspace{-1mm}
		\subsubsection{Extracting Feature Vectors}
		\label{Feature_Extraction}
		As elucidated in Sec. \ref{sec:preliminary concepts}, every input utterance, $u_{j,i}$, is associated with acoustic ($u^a_{j,i}$), visual ($u^v_{j,i}$), and textual ($u^t_{j,i}$) modals.
		We designate a modal-specific embedding to extract the feature vectors, denoted by $\vec{x}_{j,i}$, for each modal ($\vec{x}_{j,i}^m$, $m\in(a,v,t)$). The steps after feature extraction include normalization and alignment, followed by an assignment process which links features to respective utterance. Accordingly, we explain the feature extraction modules as follows:\\ 
		\textbf{Acoustic Feature Extraction. } Given the significance of acoustic expression in affective computation, we need a legitimate feature representation process. Hence, for each utterance, $u_{j,i}$ with respective prevalent textual and acoustic pre-alignments as $u^t_{j,i}$ and $u^a_{j,i}$, we partition the monophonic audio signal into textually coherent partitions. Subsequently, we adopt the acoustic means, including COVAREP \cite{Degottex2014} and OpenSmile \cite{Eyben2010a}, to exploit the audio feature vector for each partition, denoted by $\vec{x}^a_{j,i}$, including the low-level and diverse granular descriptors (LLD). To continue, we devise an acoustic-LSTM module to generate the embedded representation for each input segment, $\vec{x}^a_{j,i}$.\\
		\textbf{Visual Feature Extraction. } To infer the facial cues, we fragment the visual contents $u^v_{j,i}$ into the frames, isolate the metaphorical facial region from the background for each frame, and exclude negative frames where the human face is unobserved. To adjust the visual feature vector $\vec{x}^v_{j,i}$, we successively utilize Openface2 \cite{Baltrusaitis2018} to detect facial landmarks, estimate head pose and gaze direction, and recognize facial action units. In contrast, we apply the Facet algorithm \cite{Krizhevsky2009} to analyze the feature value distribution on various granular frame representations. The proposed Visual-LSTM generates an embedded representation for each vector $\vec{x}^v_{j,i}$.\\		
		\textbf{Textual Feature Extraction.} The textual contents, enclosed as transcripts to utterance $u^t_{j,i}$, are not infallible but tractable to exploit the textual features. A naïve solution would be to notoriously train the exact textual contents \cite{Hosseini2018}\cite{Hosseini2014}, yielding incorrect correlation weights \cite{Najafipour2020}\cite{Hosseini2020}. Alternatively, we utilize semantic embedding models \cite{Pennington2014} to obtain word vectors. We then utilize the uni-dimensional convolution networks besides the order-sensitive LSTM. Here, each convnet serves the neural network as a preprocessing tool. In other words, the convnet can transform the long input sequence into shorter sequences of higher-level features $\vec{x}^t_{j,i}$. Finally, the Textual-LSTM can feed the feature vector $\vec{x}^t_{j,i}$ to generate the associated embedded representations\vspace{-3mm}.
		\subsubsection{Cognitive Inference}
		\label{Cognitive_Inference2}
		\vspace{-1mm}
		As a complement to feature extraction, it is through mapping to cognitive space that we can hierarchically categorize the users. Here, the cognitive space ($\Omega , \digamma_\Omega^d$) adapts to a pseudo-metric $\mathbb{R}^{|C|}$ space \cite{ArkhangelSkii2012} with $\Omega$ as the set of points, $C$ as the set of cognitive cues, and $\digamma_\Omega^d$ as the pseudo-metric distance. While each cognitive cue $c^y$ represents a dimension in the $\Omega$ space, the \textit{surrogate-point} $\varphi_k=(c^1_k,\dots,c^{|C|}_k)$ delineates the user $l_k$ into $\Omega$.
		In addition, we denote the set of \textit{surrogate-points} as $\Phi$ and clarify the connotations within $\Omega$ by defining the following functions:\\
		\textit{\textbf{\textit{Mapping}}} ($\digamma_\Omega^m$: $\mathbb{L}\rightarrow\Omega$): the function mapping the user $l_k$ to a surrogate-point $\varphi_k$ in $\Omega$ using each cognitive cue $c^y_k$.\\
		\textit{\textbf{\textit{Distance}}} ($\digamma_\Omega^d$: $\Omega \times\Omega\rightarrow\mathbb{R}^+_0$): a pseudo-metric function to measure the dissimilarity ratio between each pair of points in $\Omega$, denoted by ($\varphi_i$,$\varphi_j$). Only those distance metric models that can sufficiently satisfy the pseudo-metric axioms, as formulated in eq. \ref{eq:distcond}, are eligible to estimate the similarity.\vspace{-2mm}
		\begin{equation}
			\small
			\label{eq:distcond}
			\left\{ \begin{array}{l}
			\begin{multlined}[0.5\columnwidth]Positive\;Definiteness: \digamma_\Omega^d(\varphi_i,\varphi_j) \geq 0;\end{multlined}\vspace{1mm}\\
			\begin{multlined}[0.5\columnwidth]Symmetry:  \digamma_\Omega^d(\varphi_i,\varphi_j)=\digamma_\Omega^d(\varphi_j,\varphi_i);\end{multlined}\vspace{1mm}\\
			\begin{multlined}[0.5\columnwidth]Triangle\;Inequality:\vspace{-5mm}\\ \digamma_\Omega^d(\varphi_i,\varphi_j)\leq\digamma_\Omega^d(\varphi_i,\varphi_k) +\digamma_\Omega^d(\varphi_k,\varphi_j);\end{multlined}
			\end{array}\right.
		\end{equation}
		We adopt the effective vector-based Minkowski measure as the \textit{distance} $\digamma_\Omega^d$ (Eq.\ref{eq:distcog}) that is generalizable and capable of accommodating an infinite number of cognitive cues. $\varphi_i$ and $\varphi_j$ instantiate a pair of points in cognitive space $\Omega$, and $r \in \mathbb{N}$ denotes the adjustment parameter\vspace{-2mm}.
		\begin{equation} 
			\small
			\label{eq:distcog}
			\digamma_\Omega^d(\varphi_i,\varphi_j)=\displaystyle(\sum_{d=1}^{|C|}(|c_i^d-c_j^d|)^r)^{\dfrac{1}{r}}
			\vspace{-2mm}
		\end{equation}
		Successively, we propose the mapping functionality $\digamma_\Omega^m$ by an effective model \cite{Guecluetuerk2017} that consumes the multimodal users' data to predict multi-aspect character attributes. To continue, we propose an observation on user mapping using three cues of agreeableness, consciousness, and extroversion. As Fig. \ref{fig:clustering3dspace} depicts, we appoint each user in the dataset \cite{Zadeh2016} to a surrogate point in $\Omega$ using a three-dimensional weight. Subsequently, we can utilize the K-means method for clustering, illustrated in various colors.
		\begin{figure}[htp]
			\vspace{-5mm}
			\centering
			\includegraphics[width=0.65\linewidth]{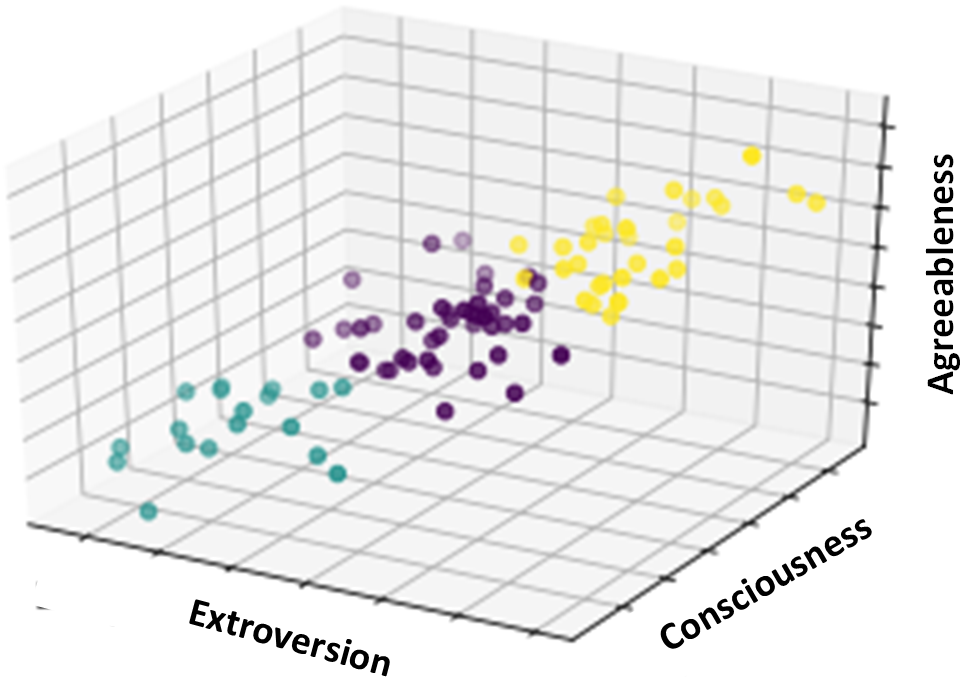}
			\vspace{-3mm}
			\caption{User distribution in a three-dimensional $\Omega$ on MOSI\cite{Zadeh2016}}
			\label{fig:clustering3dspace}
			\vspace{-5mm}
		\end{figure}
		\begin{lemma}
			\textit{The cognitive space is a pseudo-metric space.}
			\label{lemma_nq} 
			\vspace{-2mm}	
		\end{lemma}
		\begin{proof}
			By a contradiction, we assume $(\Omega,\digamma^d_\Omega)$ is a metric space. In response, the distance module $\digamma^d_\Omega$ must preserve the metric axioms, including the non-zero criterion for the distance between any pair of specific points (Eq. \ref{eq:metriccond})\vspace{-2mm}.
			\begin{equation}
				\small
				\label{eq:metriccond}
				\forall \varphi_i,\varphi_j\in\Omega:\; \varphi_i\neq\varphi_j\Longleftrightarrow \digamma^d_\Omega(\varphi_i,\varphi_j)>0
				\vspace{-2mm}
			\end{equation}
			Let $l_i$ and $l_j$ be a pair of users that can have identical cognitive cues. Inherently, due to the high cognitive correlation as formalized in Eq. \ref{eq:lemm1}, $\digamma_\Omega^m$ can map the pair to the surrogate-points of the same coordination in $\Omega$, reflecting the zero distance between users. Such an outcome can contradict the primary assumption, proving the proposition.
			\vspace{-2mm}
			\begin{equation}
				\small
				\label{eq:lemm1}
				\begin{multlined}[0.5\columnwidth]
				\exists l_i, l_j, l_i\neq l_j\; and \; \underset{c_y \in C}{\forall}( \digamma_\Omega^m(l_i).c_y=\digamma_\Omega^m(l_j).c_y) \vspace{-2mm} \\ \Longrightarrow \digamma_\Omega^d(\digamma_\Omega^m(l_i),\digamma_\Omega^m(l_j))=0
				\end{multlined}
			\vspace{-5mm}
			\end{equation}
			\vspace{-2mm}
		\end{proof}
		\begin{lemma}
			\textit{Minkowski measure satisfies pseudo-metric axioms.}
			\vspace{-2mm}
			\begin{equation}
				\small
				\label{eq:lem2normdist}
				\left\{ \begin{array}{l}
					\digamma^d_\Omega(\varphi_i,\varphi_k)=||u+v||\\
					\digamma^d_\Omega(\varphi_i,\varphi_j)=||u||\\
					\digamma^d_\Omega(\varphi_j,\varphi_k)=||v||
				\end{array}\right.
				\vspace{-2mm}
			\end{equation}	
		\end{lemma}
		\begin{proof}
			With axioms in Eq. \ref{eq:distcond}, the proof to preserve the Triangle Inequality can leave out the naïve proofs about both trivial features of the Minkowski metric, positive definiteness and the symmetry criterion. Accordingly, let $\varphi_i,\:\varphi_j,\:\varphi_k\in \Omega$ be a triple surrogate-points and define the vectors $u$ and $v$ as $u=\varphi_i-\varphi_j$ and $v=\varphi_j-\varphi_k$, with the sum vector as $u+v=\varphi_i-\varphi_k$. Hence, the vector norms in Eq. \ref{eq:lem2normdist} equate to the distance between corresponding points. Given $u.v$ as the dot product of the pair of vectors $u$ and $v$, we can adopt the norm convention to affirm the symmetry in Eq. \ref{eq:lem2eq1} \vspace{-2mm}.
			\begin{equation}
				\small
				\label{eq:lem2eq1}
				||u+v||^2=||u||^2+||v||^2+2 u.v
				\vspace{-2mm}
			\end{equation}
			Given the triangle axiom, we must prove Eq. \ref{eq:lem2eq2} inequality\vspace{-2mm}:
			\begin{equation}
				\small
				\label{eq:lem2eq2}
				\digamma^d_\Omega(\varphi_i,\varphi_k)\leq\digamma^d_\Omega(\varphi_i,\varphi_j)+\digamma^d_\Omega(\varphi_j,\varphi_k)
				\vspace{-2mm}
			\end{equation}
			To commence, we can attain the sequential order in Eq. \ref{eq:lem2ineq} by the criteria verbalized in Eq. \ref{eq:lem2normdist}\vspace{-2mm}:
			\begin{equation}
				\small
				\label{eq:lem2ineq}
				\begin{multlined}[0.5\columnwidth]||u+v||\leq ||u||+||v||\overset{Square}{\Longrightarrow}\\
					||u+v||^2\leq ||u||^2+||v||^2+ 2||u||\times||v||
				\end{multlined}
				\vspace{-2mm}
			\end{equation}
			By replacing the Eq. \ref{eq:lem2eq1} in Eq. \ref{eq:lem2ineq}, we can retrieve Eq. \ref{eq:lem2eq3}\vspace{-2mm}.
			\begin{equation}
				\small
				\label{eq:lem2eq3}
				\begin{multlined}[0.5\columnwidth]
					||u||^2+||v||^2+2 u.v\leq ||u||^2+||v||^2+ 2||u||\times||v||\\
					\Longrightarrow u.v\leq||u||\times||v||
				\end{multlined}
				\vspace{-2mm}
			\end{equation}
			Here the Eq. \ref{eq:lem2eq3} results in Cauchy-Schwarz inequality\cite{Ge2013} ($|u.v|\leq||u||\times||v||$) that logically proves the lemma.
		\end{proof}
		\vspace{-4mm}
		\subsubsection{Adaptive Ensemble Tree Definition}
		\label{ensmebletreedefinition}
		\vspace{-1mm}
		The necessity to devise a novel tree structure to categorize users through performing sharding on the $\Omega$ space sets off before determining the sentiment prediction models. Aiming to carry out the user grouping procedure, we rely on the verified association between the individuals' latent cues and affective manifestation \cite{Ngai2015}.\\	
		Inspired by VP-Tree \cite{Malkov2018}, which searches more efficiently compared to R*-tree \cite{Taha2015}, we develop a non-trivial adaptive tree $T=(N,E)$ with $N$ and $E$ as the set of nodes and edges. In response, the tree utilizes surrogate points to divide $\Omega$ space based on the dimensions. Accordingly, the sharding procedure hierarchically constructs a collection of subspaces from $\Omega$, facilitating parallel processing in simultaneous training of multiple subspaces.\\
		The properties of each node $n$ can be threefold: (1) The \textit{Subspace} $n^\Phi$ comprising a set of surrogate points created by sharding on the parent node, (2) The \textit{Submodel} $n^m$, a transfer-based LSTM model that consumes the videos in $n^\Phi$, denoted by $\{\digamma^v_l(l_k)|l_k\in n^\Phi\}$, and (3) The \textit{Knowledge Core} $n^k$ to gain an aggregated perception from the child nodes. Furthermore, the adaptive tree must hold two functions\vspace{-2mm}: 
		\begin{itemize}
			\item\textbf{\textit{Fragmentation}} ($\digamma_T^f: N\rightarrow \Omega$)
			divides the subspace pertinent to the given node $n^\Phi$ into $k$ tractable clusters, adjusted by parameter $\ddot{\rho}$ where the cluster set should attain specific criteria denoted by $\gamma^\Phi$\vspace{-2mm}:
			\begin{equation}
				\small
				\label{eq:nodefragmantationcond}
				\gamma^\Phi=\left\{ \begin{array}{l}
				\begin{multlined}[0.5\columnwidth]\Phi_i\neq\emptyset,\; i\leq k; \end{multlined}\\
				\begin{multlined}[0.5\columnwidth]\cup_{i=1}^k\Phi_i=n_i^\Phi;\end{multlined}\\
				\begin{multlined}[0.5\columnwidth]\Phi_i\cap\Phi_j=\varnothing,\;i,j\leq k\;\&\;i\neq j;\end{multlined}\vspace{-1mm}\\
				\end{array}\right.
				\vspace{-1mm}
			\end{equation}
  			Here $\Phi_i$ and $\Phi_j$ are the splits from fragmenting the node subspace $n^\Phi$. Union ($\cup$) ensures the subspace $n^\Phi$ will be a portmanteau of resultant splits and intersection ($\cap$) affirms the exclusivity of the splits.
			\item\textbf{\textit{Disjunction}} ($\digamma_T^d:\Omega\rightarrow\mathbb{R}^+_0$) specifies the relative distance between every pair of node subspaces in the same level of the adaptive tree (Eq.\ref{eq:nodedist}). For ($n_i$,$n_j$) as a pair, the disjunction will be the cognitive \textit{distance} $\digamma_\Omega^d$ between the medoids, denoted by $\tilde{\varphi}$:\vspace{-2mm}
			\begin{equation}
				\small
				\label{eq:nodedist}
				\digamma_T^d(n_i^\Phi,n_j^\Phi) = \digamma_\Omega^d(n_i^\Phi.\tilde{\varphi},n_j^\Phi.\tilde{\varphi})
				\vspace{-2mm}
			\end{equation}
		\end{itemize}
		\subsubsection{Cognitive space categorization}
		\label{sec:Cognitive-space-categorization}
		In this step, we devise a modified top-down breadth-first traversing method to construct the adaptive tree (Sec. \ref{ensmebletreedefinition}) by iteratively partitioning the $\Omega$ space. Given node $n_i$, we can terminate the node branching by controlling the triple data attributes of \textit{cohesion}, \textit{sparsity}, and \textit{amplitude}. Given $H$  as the set of fragments on the node $n_i$, we define trilateral termination criteria for each split $\Phi\;\in H$ as follows: \\
		\textbf{Criterion1.} Subject to cohesiveness attribute for each split $\Phi$, Eq. \ref{eq:procrcohesive} determines the \textit{unity} measure $\digamma^{uni}$ to compute the cognitive \textit{distance} $\digamma_\Omega^d$ between each pair of points in $\Phi$\vspace{-2mm}.
		\begin{equation}
			\small
			\label{eq:procrcohesive}
			\digamma^{uni}(\Phi)=\left( \begin{array}{c} |\Phi| \\ 2 \end{array} \right)^{-1} \underset{\varphi_x,\varphi_z \in \Phi}{\sum}\digamma_\Omega^d(\varphi_x,\varphi_z)
			\vspace{-2mm}
		\end{equation}
		Eq. \ref{eq:recuperation} obtains the branching tendency on the node subspace $n_i^\Phi$ for each split by evaluating the quality of fragmentation using \textit{unity} and \textit{disjunction} trade-offs. Hence, we collectively analyze the mean-variance (MV) to compute the \textit{recuperation} capacity $\digamma^{rcp}$ using the acquired measures\vspace{-2mm}.
		\begin{equation}
			\small
			\label{eq:recuperation}
			\digamma^{rcp}(H)=\underset{\Phi \in H}{MV}(\dfrac{(|H|-1)\times\digamma^{uni}(\Phi)}{\underset{\Phi_j \in H,\Phi_j\neq\Phi}{\sum}\digamma_T^d(\Phi,\Phi_j)})
			\vspace{-2mm}
		\end{equation}
		Where lower values of recuperation ensure better cohesion and separation, higher values than $\theta_p$ can stop branching.\\
		\textbf{Criterion2.} Besides the recuperation ratio, we also need to ensure the training capability of independent splits. Hence, we propose an entropy-based criterion that consumes the sentiment labels within split data to control the bias. Correspondingly, Eq. \ref{eq:pertUtterances} formalizes the set of utterances $U_j$ pertinent to users of a subspace $\Phi_j$\vspace{-2mm}.
		\begin{equation}
			U_j=\{u_{z,x}|u_{z,x}\subseteq v_x,v_x\in\digamma_l^v(l_k), l_k\in\Phi_j\}
			\vspace{-2mm}
			\label{eq:pertUtterances}
		\end{equation}
		Eq. \ref{eq:entropyprop} can estimate the portion of utterances $U_j$ convey the sentiment $\tilde{s}$ in each split $\Phi_j$\vspace{-2mm}.
		\begin{equation}
			\small
			\label{eq:entropyprop}
			\digamma^{pr}(\tilde{s},\Phi_j)=\{\dfrac{|u|}{|U_j|} | u \in U_j, u.s=\tilde{s}\}
			\vspace{-2mm}
		\end{equation}
		Eq. \ref{eq:shannonentropy} adopts Shannon entropy \cite{Li2020} to measure data impurity in the split $\Phi_j$\vspace{-2mm}.
		\begin{equation}
			\small
			\label{eq:shannonentropy}
			\digamma^{imp}(\Phi_j) = -\underset{\tilde{s} \in \{0,1\}}{\sum} \digamma^{pr}(\tilde{s},\Phi_j) \log(\digamma^{pr}(\tilde{s},\Phi_j))
			\vspace{-2mm}
		\end{equation}
		We can utilize the impurity ratio to evaluate the quality of the set $H$ of splits by ${\biguplus}_{\Phi_j\in H}\digamma^{imp}(\Phi_j)$, where $\biguplus$ denotes the weighted average. Accordingly, we can compare the result with the threshold $\theta_e$ to terminate nodes with lower rates.\\	
		\textbf{Criterion3.} Relying on data amplitude, this restrain asserts under what verge the training can infer meaningful information from each of the splits. Given the set of pertinent utterances $U_j$ in Eq. \ref{eq:pertUtterances} for a split $\Phi_j$, we ensure the amplitude capacity of the splits using threshold $\theta_a$, a prerequisite to terminate the cases with lower values than ${min}_{\Phi_j\in H}\{|U_j|\}$.\vspace{1mm}\\
		Algorithm \ref{alg:gentreegeneration} initializes the adaptive tree $T=(N,E)$, the cognitive aggregate learning machine, with $N$ and $E$ as the set of nodes and edges. Accordingly, the fragmentation module can extend the tree by adopting clustering and theoretical-based categorization techniques. We utilize a breadth-first iterative fragmentation procedure to generate the subspaces for each node. Thus, we insert the subspaces in-row with parents into the allocated queue, $Q$, where the genuine $\Omega$ space represents the root node of no parent.\\ 
		In a FIFO manner, we pop the item $x$ in each iteration from the queue with pertinent parent and subspace. We then initialize $x$ as node $n$ and append it to the intrinsic tree. Additionally, we \textit{fragment} the node $n$ subspace by adjusting an optimized $\ddot{\rho}$ to retrieve $H$ as the set of child splits. To continue, we apply trilateral termination criteria on $H$ to evaluate branching and either append the node $n$ as a leaf or insert each split of $H$ with $n$ as a parent into the queue, resulting in the growth of the subspaces in $Q$.			
		\vspace{-3mm}
		\begin{algorithm}[H]
			\caption{Adaptive Tree Generation}
			\label{alg:gentreegeneration}
			\textbf{Input:} $\Omega$\\
			\textbf{Output:} $T,Leaf$
			\begin{algorithmic}[1]
				\STATE $Leaf=\emptyset, N=\emptyset, E= \emptyset, Q=\emptyset $
				\STATE $Q.Enque(\{\Omega,\emptyset\})$
				\WHILE{$Q \neq \emptyset$}
				\STATE $\;x=Q.Deque()$
				\STATE $n=Node(x[0],\emptyset,\emptyset)$
				\STATE $N.append(n)$
				\STATE $E.append(\{n,x[1]\})$
				\STATE $\ddot{\rho}=AdaptParameters(x[0])$
				\STATE $H=NodeFragmentation(n,\ddot{\rho})$
				\IF{$Terminate(H)$}
				\STATE $Leaf.append(n)$
				\ELSE
				\FORALL{h in H}
				\STATE $Q.Enque(\{h,n\})$
				\ENDFOR
				\ENDIF
				\ENDWHILE
				\STATE $T=(N,E)$
				\STATE return $T$, Leaf
			\end{algorithmic}
			\vspace{-1mm}
		\end{algorithm}
		\vspace{-3mm}
		\noindent Prior to the tree construction, the fragmentation procedure divides the subspace of each vertex $n$ into a set of tractable clusters, to be handled in two ways as follows:\vspace{2mm}\\
		\textbf{Theoretical-based Subspace Categorization.} In this approach, we employ the fragmentation $\digamma_T^f$ to build the tree $T$ within $\Omega$ space using the cut-points as $\Omega$’s dimensions. Accordingly, $\digamma_T^f$ can consume the parameter $\ddot{\rho}$ comprising the pair \{$c^*,c^*_{cut}$\} to fragment given node on the $c^*$ dimension to maximize the heterogeneity with $c^*_{cut}$ as the optimal cut-point. Having $C$ as the set of dimensions, Eq.\ref{eq:dimentionselection} verbalizes the dimension $c^*$ selection for the fragmentation of node $n$\vspace{-1mm}.
		\begin{equation}
			\small
			\label{eq:dimentionselection}
			c^*=\underset{c\in C}{\arg \max}\{\digamma^{het}(n^\Phi,c)\}
			\vspace{-3mm}
		\end{equation}
		Inspired by \cite{Basak2005}, the ratio $\digamma^{het}$ retrieves the heterogeneity by consuming the subspace $\Phi$ and the dimension $c$ (Eq.\ref{eq:hetromeasure})\vspace{-2mm}.
		\begin{equation}
			\small
			\label{eq:hetromeasure}
			\digamma^{het}(\Phi,c)=\underset{\varphi_x,\varphi_z \in \Phi} {-\sum}\widetilde{d}_{x,z}(1-\widetilde{d}_{x,z}^c)+\widetilde{d}_{x,z}^c(1-\widetilde{d}_{x,z})
			\vspace{-3mm}
		\end{equation} 
		Here $\widetilde{d}_{x,z}$ equates the local normalized cognitive distance  $\digamma^d_\Omega(\varphi_x,\varphi_z)\times d_{max}^{-1}$ with $d_{max}$ as the maximum local distance (i.e., ${Max}_{\varphi_x, \varphi_z\in \Phi}\{\digamma^d_\Omega(\varphi_x,\varphi_z)\}$), and $\widetilde{d}^c_{x,z}$ is calculated in the shrunk $\Omega$ space by the dimension $c$. To continue, we can use Eq.\ref{eq:cut-pointselection} to assert fragmentation with the best cut-point $c^*_{cut}$ that, on the one hand, maintains cohesiveness within each of the splits, and on the other hand, ensures unbiased labels in each resultant subspaces\vspace{-2mm}. 
		\begin{equation}
			\small
			\label{eq:cut-pointselection}
			c^*_{cut}=\underset{k\in c^*}{\arg \max}\{ \lambda\underset{\Phi\in H_k}{\biguplus}\digamma^{het}(\Phi,c^*)+(1-\lambda)\underset{\Phi\in H_k}{\biguplus}\digamma^{imp}(\Phi) \}
			\vspace{-2mm}
		\end{equation}
		Here $\lambda$ is the adjustment parameter and $H_k$ denotes the split pair formed by cutting the node $n^\Phi$ using the cut-point $k$.\vspace{2mm}\\
		\textbf{Cluster-based Subspace Categorization.}  In this approach, we adopt the simple but effective K-means clustering as $\digamma_T^f$ to fragment the subspaces. Unlike the theoretical-based model, K-means comprises all dimensions in partitioning, resulting in a more flexible sub-spacing. As verbalized in Eq.\ref{eq:clustering}, we consume the parameter setting $\ddot{\rho}$ of $\{d^*\}$ on node $n$ to categorize the node subspace $n^\Phi$ into $d^*$  clusters\vspace{-2mm}.
		\begin{equation}
			\small
			\label{eq:clustering}
			\digamma_T^f(n,\{d^*\})=\underset{R}{\arg\min}\{\overset{|n^\Phi|}{\underset{x=1}{\sum}}\overset{d^*}{\underset{y=1}{\sum}}r_{xy}\times\digamma_\Omega^d(\varphi_x,\Phi_y.\tilde{\varphi}) \}
			\vspace{-2mm}
		\end{equation}
		Here R=($r_{xy}$) is an $|n^\Phi|\times d^*$ matrix, $r_{xy}\in\{0,1\}$ estimates if user $x$ participates in cluster $\Phi_y$ with $\Phi_y.\tilde{\varphi}$ as medoid, where the matrix assures exclusivity in fragmentation conditioned by $\sum^{d^*}_{y=1}r_{xy}=1$ and $0<\sum^{|n^\Phi|}_{x=1}r_{xy}<|n^\Phi|$. Finally, Eq. \ref{eq:clusternum} computes the ideal number of clusters, denoted by $d^*$.\vspace{-1mm}
		\begin{equation}
			\small
			d^*=\underset{d}{\arg\max}\{\underset{\Phi\in H_d}{\biguplus}\digamma^{slh}(\Phi)+\underset{\Phi\in H_d}{\biguplus}\digamma^{imp}(\Phi)-\digamma^{rcp}(H_d) \}
			\label{eq:clusternum}
			\vspace{-2mm}
		\end{equation}
		Here $H_d$ is the set of splits generated by $\digamma_T^f(n_i,\{d\})$ and $\digamma^{slh}$ is the silhouette measure \cite{Iglesias2019}.
		\begin{figure}[H]
			\centering
			\begin{subfigure}{0.49\linewidth}
				\centering
				\includegraphics[width=\textwidth]{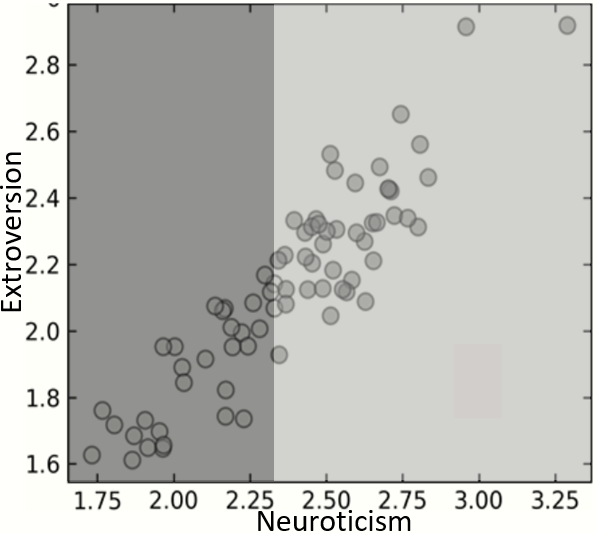}
				\vspace{-5mm}
				\caption{Theoretical-based}
				\label{fig:INF-cat}
			\end{subfigure}
			\centering
			\begin{subfigure}{0.49\linewidth}
				\centering
				\includegraphics[width=\textwidth]{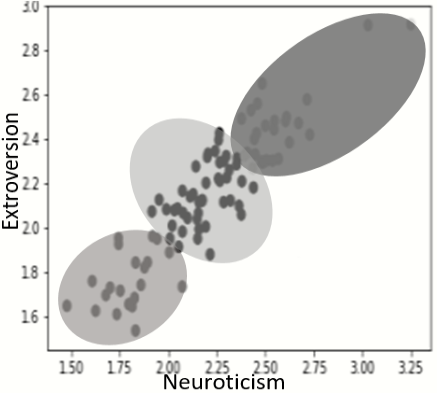}
				\vspace{-5mm}
				\caption{Clustering-based}
				\label{fig:CHF-cat}
			\end{subfigure}
			\hfill
			\vspace{-3mm}
			\caption{Proposed Categorizations}
			\vspace{-3mm}
			\label{fig:ObsCat}
		\end{figure}
		\noindent Fig. \ref{fig:ObsCat} compares the theoretical-based categorization versus clustering. Given extroversion and neuroticism cognitive cues, the $\Omega$ space can be of two dimensions. While the theoretical-based approach initially divides the space into two splits using respective cut-point on neuroticism, the clustering-based approach groups the surrogate-points into three clusters\vspace{-1mm}.    
		\subsubsection{Transfer-Based Submodels and Knowledge Core}
		In this section, we utilize a bottom-up traversing approach to concurrently construct the transfer-based submodels and the knowledge cores. As visualized in Fig. \ref{fig:treearch}, for each node $n$, on the one hand, we form the knowledge core $n^k$ by consuming the submodel output of the child nodes, and on the other hand, we collectively feed the submodel $n^m$ with the pertinent video contents of the users belonging to the node subspace and the knowledge core $n^k$. Hence, the tree can meticulously transfer the knowledge from child submodels to the corresponding parent submodels.
		\begin{figure}[h]
			\centering
			\vspace{-4mm}
			\includegraphics[width=1\linewidth]{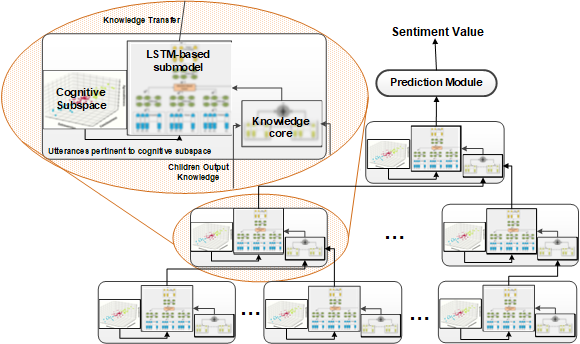}
			\vspace{-6mm}
			\caption{The Proposed Adaptive Tree Architecture}
			\label{fig:treearch}
			\vspace{-4mm}
		\end{figure}
	
		\noindent As a prerequisite to generating the node-specific submodels, we pad the utterance set of each video to a maximum length $g$ as formulated in Eq. \ref{eq:MaxUttPerVid}\vspace{-3mm}:
		\begin{equation}
			\small
			\label{eq:MaxUttPerVid}
			g = \underset{v_i\in\mathbb{V}}{\max} \digamma_v^{len}(v_i)
			\vspace{-2mm}
		\end{equation}
		Given node $n$, each submodel $n^m$ includes three unimodal LSTMs for each input modal and a multimodal LSTM. The training of $n^m$ conjointly conveys both the node pertinent videos $\{v_i|v_i\in\digamma_l^v(l_k),l_k\in n^\Phi\}$ and the knowledge core $n^k$. Given an input video $v_i$ with the set of utterances $[u_{1,i},u_{2,i},\dots,u_{g,i}]$, we define the context criterion for $u_{i,j}$ as $\{u_{i,z}|\forall z<g , z\neq j\}$. As Fig. \ref{fig:deeparch} visualizes the submodel architecture for $n^m$, given feature vectors $X^m=\{x^m_t|x^m_t\in\mathbb{R}^{d_m},t\leq g\}$ with a set of modals as $m\in(A,V,T)$ and $d_m$ as the length of the corresponding feature vector, the unimodal models can generate modal-specific embeddings, represented by $H^m=\{h^m_t|h^m_t\in\mathbb{R}^{\tilde{d}},t\leq g\}$. The resultant embeddings $H^m$ and the node knowledge core $K=\{k_t|k_t\in \mathbb{R}^{\tilde{d}},i\leq g \}$ collectively feed into the attention-based fusion module to provide the fused input of $H^F=\{h^F_t|h^F_t\in\mathbb{R}^{\tilde{d}},t\leq g\}$ for the final multimodal model. To conclude, the multimodal model utilizes $H^F$ to transfer an enhanced knowledge, denoted by $H^k=\{h^k_t|h^k_t\in\mathbb{R}^{\tilde{d}},t\leq g\}$, to the parent knowledge core.
		\begin{figure}[h]
			\centering
			\vspace{-5mm}
			\includegraphics[width=1\linewidth]{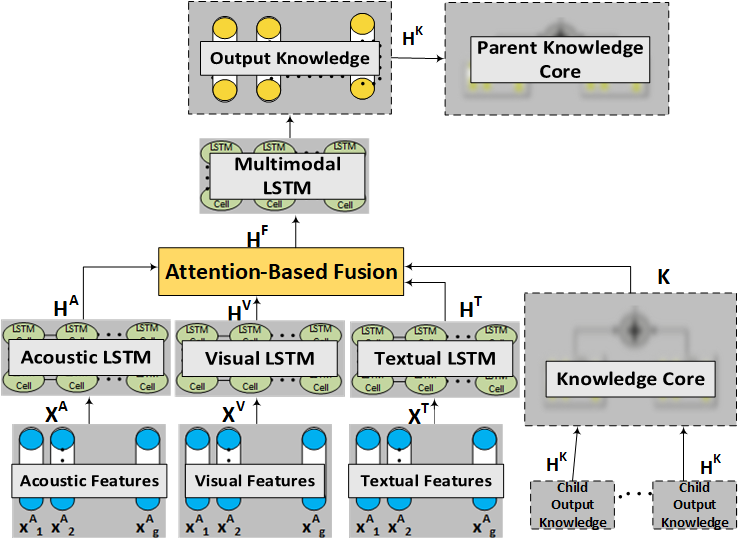}
			\vspace{-6mm}
			\caption{Submodels architecture}
			\label{fig:deeparch}
			\vspace{-3mm}
		\end{figure}
		\par \noindent For both unimodal and multimodal models, INLSTM: $\mathbb{R}^{d_m \times g}\rightarrow\mathbb{R}^{\tilde{d} \times g}$ as an LSTM procedure can convey sentiment analysis. As formalized in Eq. \ref{eq:lstmformula}, besides the internal state $s_t^m$, and the output $h_t^m$, the INLSTM comprises three gates of forget $f_t^m$, external-input $q_t^m$, and output $o_t^m$\vspace{-2mm}.
		\begin{equation*}
			\small
			f_t^m = \S(U_f^m.X^m + W_f^m.h_{t-1}^m + b_f^m) 
			\vspace{-1mm}
		\end{equation*}
		\begin{equation*}
			\small
			q_t^m = \S(U_e^m.X^m + W_e^m.h_{t-1}^m + b_e^m) 
			\vspace{-1mm}
		\end{equation*}
		\begin{equation*}
			\small
			o_t^m = \S(U_o^m.X^m + W_o^m.h_{t-1}^m + b_o^m) 
			\vspace{-1mm}
		\end{equation*}
		\begin{equation*}
			\small
			s_t^m = f_t^m \odot s_{t-1}^m + q_t^m \odot tanh(U_s^m.X^m + W_s^m.h_{t-1}^m + b_s^m)
			\vspace{-1mm}
		\end{equation*}
		\begin{equation}
			\small
			\label{eq:lstmformula}
			h_t^m= o_t^m \odot tanh(s_t^m)
			\vspace{-1mm}
		\end{equation}
		Here, $W^m_f$, $W^m_e$, $W^m_o$, $W^m_s$, $U^m_f$, $U^m_e$, $U^m_o$, $U^m_s \in\mathbb{R}^{d_m\times2d_m} , b_f^m$, $b_c^m$, $b_o^m$, $b_s^m \in \mathbb{R}^{d_m}$ are training parameters where $\S$ is the sigmoid function and $\odot$ denotes the element-wise multiplication.
		Given $m\in (A,V,T)$, the INLSTM can generate the pertinent embedding output for the given channel as $H^m$. Aiming to carry out the fusion procedure (Eq. \ref{eq:attentionformula}), we concatenate INLSTM outputs with the knowledge core $K$ to form $H=\{H^A||H^V||H^T||K\}$ that feeds the multimodal input for the attention-based module\vspace{-2mm}.
		\begin{equation*}
			\small
			b_t = tanh(W_b.H)
			\vspace{-1mm}
		\end{equation*}
		\begin{equation*}
			\small
			\alpha_t = \sigma(\omega_b^T.b_t)
			\vspace{-1mm}
		\end{equation*}
		\begin{equation*}
			\small
			c_t = H.\alpha_t^T
			\vspace{-1mm}
		\end{equation*}
		\begin{equation}
			\small
			\label{eq:attentionformula}
			h_t^F = tanh(W_h.c_t + U_h.H)
			\vspace{-2mm}
		\end{equation}
		Having $\sigma$ as the softmax function, we can retrieve $H^F$ as an attention-based fusion for the output of unimodal layers and the knowledge core to train the multimodal layer within INLSTM, resulting in $H^K$. Finally, we can transfer the knowledge from $H^K$ to feed the knowledge core of the parent node, continuing the bottom-up learning procedure toward the root in the hierarchy\vspace{-4mm}.
		\subsubsection{Constructing Prediction Module}
		\label{sec:constructpredictionmodule}		
		In general, our proposed framework maps the dataset users to the $\Omega$ space, resulting in an adaptive tree that can hierarchically categorize the users and learn the ensemble of submodels. The prediction module can estimate the sentiments by consuming the final improved representation of the data in the root. As formalized in Eq. \ref{eq:sigma_W_g}, the prediction module can utilize a softmax function on $H_0^K=[h_{0,1}^K,h_{0,2}^K,\dots,h_{0,g}^K]$ with $h_{0,t}^K\in\mathbb{R}^{\tilde{d}}$ that is the output of the root node submodel\vspace{-2mm}.
		\begin{equation}
		\label{eq:sigma_W_g}
			\small
			g_t = \sigma(W_g.(h_{0,t}^K)^T + b_g) 			
			\vspace{-2mm}
		\end{equation}
		\begin{equation*}
		\small
		\label{eq:predictionmodule}
		y_t' = \underset{m\in [0,1]}{\arg\max}(g_t[m])
		\vspace{-2mm}
		\end{equation*}
		Here, $W_g, b_g$ are parameters we learn during training the prediction layer, and $Y'=\{y_t'|y_t'\in\{0,1\}, t\leq g\}$ is the final set of sentiment predictions for each utterance where $H_0^K$ is normalized to a probability distribution over binary sentiments using cognitive-based perceptions\vspace{-4mm}. 	
		\subsubsection{Kind Neighbors}
		\label{kind-neighbors}
		\vspace{-1mm}
		Given the node submodels with inadequate users, our goal in this section is to reduce data sparsity. Hence, we employ a non-trivial approach to compensate insufficiency by borrowing similar users from neighboring nodes. Each user within a node can then either be original or marked as adopted from adjacent nodes. Consequently, we propose the dynamic dropout to optimize submodel training.\\
		The sparsity phenomenon occurs when $|n_o^\Phi|$, as the number of users in the subspace of the node $n_o$, deems to be less than threshold $\theta_b$. Accordingly, we can locate the nearest neighbors using the least distance between pairs (Eq. \ref{eq:nodedist}) and continue adopting the users from neighbors as long as the amplitude satisfies the $\theta_b$ constraint. Hence, the higher the dissimilarity between the nodes, the lower the node adaptation chance, and the higher the dropout ratio will be, resulting in a more generalized model in training. Given the set of data enclosed with the adopted users from a neighboring node $n_c$, Eq. \ref{eq:dropprob} calculates the dropout ratio $p^d_{c,o}$ for the given $n_c$ to facilitate training of the submodel of $n_o$\vspace{-2mm}.
 		\begin{equation}
			\small
			\label{eq:dropprob}
			p^d_{c,o}=\dfrac{\underset{\varphi_i,\varphi_j\in n_o^\Phi}{Max}\{\digamma_\Omega^d(\varphi_i,\varphi_j)\}}{\digamma_T^d(n_c,n_o)}\times\lambda
			\vspace{-2mm}
		\end{equation} 
		Here, $\lambda$ is the adjusting parameter to scale the dropout probability. Accordingly, the generalization of the model will be proportionate to the conditional likelihood of the adoption cases for each given original node. We compute the probability by dividing the maximum distance between the surrogate points within $n_o$ sub-space by the distance between $n_o$ and any given neighboring node $n_c$\vspace{-5mm}. 		
		\subsection{Online Phase}
		\vspace{-1mm}
		In the online phase, we evaluate the output sentiment for the given video using a cognitive-based ensemble implemented within an adaptive tree. After leveraging the features of all data channels, we accomplish three procedures to complete the inference process: Cognitive Assessment, Cognitive Tree Routing, and Cognitive-Based Prediction\vspace{-3mm}.
		\subsubsection{Cognitive Assessment}
		\label{sec:Cognitive-Assess}
		In this section, given the input video and the participating user, respectively denoted by $v$ and $v.l$, we utilize the function $\digamma_\Omega^m$ to map the user $v.l$ to a corresponding surrogate-point $\varphi_\bullet$ within the $\Omega$ space (Eq. \ref{digamma_Omega_mv_l})\vspace{-2mm}.
		\begin{equation}
		    \label{digamma_Omega_mv_l}
			\varphi_\bullet=\digamma_\Omega^m(v.l)
			\vspace{-2mm}
		\end{equation}
		This approach is beneficial for cold-start users, where we exploit the cognitive cues and estimate sentiments simultaneously. Moreover, we can benefit from mapping to track a user-specific route within the adaptive tree\vspace{-3mm}.
		\subsubsection{Adaptive Tree Routing}
		\label{sec:CognitiveTreeRouting}
		In this section, we find the most relevant leaf node to the corresponding surrogate point. To this end, we traverse a route from the root toward the objective leaf node in the adaptive tree. Following the categorization approach (Sec. \ref{sec:Cognitive-space-categorization}) considered in the construction process of the tree, for each node $n$ within the traverse procedure, we obtain the similar child node $\ddot{n}$ to the surrogate-point $\varphi_\bullet$.\\
		\textbf{Theoretical-Based Categorization} In this approach, as a prerequisite of fragmentation, for each node $n$ in the adaptive tree, we hold both dimension $c^*$ and corresponding cut-point $c^*_{cut}$ to assign the data to the left or right child node. We then compare the value of $c^*$ at the point $\varphi_\bullet$ versus $c^*_{cut}$ to choose the corresponding child node $\ddot{n}$.\\
		\textbf{Clustering-Based Categorization} Given $N^c$ as the set of child nodes, the clustering-based method computes the distance between $\varphi_\bullet$ and the medoid of the subspace comprising candidate nodes to determine $\ddot{n}$\vspace{-2mm} (Eq. \ref{eq:selectleafcluster}).	
		\begin{equation}
			\small
			\label{eq:selectleafcluster}
			\ddot{n}=\underset{n_\kappa\in {N^c}}{\arg\min}\{\digamma_\Omega^d(n_\kappa^\Phi.\tilde{\varphi},\varphi_\bullet)\}
			\vspace{-2mm}
		\end{equation} 
		Here $n_\kappa^\Phi.\tilde{\varphi}$ is the subspace medoid of $n_\kappa$. In a nutshell, we firstly traverse the adaptive tree from the root toward the specific leaf node in an iterative top-down process and subsequently span the retrieved path in a bottom-up manner to execute the transfer-based submodels\vspace{-4mm}.
		\subsubsection{Cognitive-Based Prediction}
		\vspace{-1mm}
		Algorithm \ref{alg:Predicttestdata} estimates cognitive-based sentiments for the input video $v$, comprising four steps: feature extraction, cognitive assessment, adaptive tree routing, and prediction.
		\vspace{-6mm}
		\begin{algorithm}[H]
			\caption{Prediction In Online Phase}
			\label{alg:Predicttestdata}
			\textbf{Input:} {v}\\
			\textbf{Output:} {s}
			\begin{algorithmic}[1]
				\STATE $X=\emptyset, H=\emptyset, S=\emptyset, P=\emptyset $
				\FORALL{$m\:in\:\{A,V,T\}$}
				\FORALL{$u_j\in v$}
				\STATE $x_j^m = ExtractFeatures(u_j^m$)
				\ENDFOR
				\ENDFOR
				\STATE $\varphi_\bullet=\digamma_\Omega^m(v.l)$
				\STATE $\ddot{n}=n_0$
				\WHILE{$\ddot{n}\neq\:Null$}
				\STATE $S.Push(\ddot{n})$
				\STATE $\ddot{n}=SelectChildNode(\ddot{n},\varphi_\bullet)$ 
				\ENDWHILE
				\WHILE{$S\neq\emptyset$}
				\STATE $\ddot{n}=S.Pop()$
				\STATE $\ddot{n}^k.append(H)$
				\STATE $H = \ddot{n}^m(\{X^A,X^V,X^T\},\ddot{n}^k)$
				\ENDWHILE
				\STATE $Y' = Predict(H)$ \COMMENT{Eq. \ref{eq:predictionmodule}}
				\STATE $s = MajVote(Y')$
				\STATE return s
			\end{algorithmic}
		\end{algorithm}
		\vspace{-4mm}
		\noindent Given video $v$ with the set of utterances $v=[u_1,\dots,u_g]$ padded to the length $g$ (Eq. \ref{eq:MaxUttPerVid}), we first extract the modal-specific feature vectors $X^m\in\mathbb{R}^{d_m\times g}$ for all utterances, with $m\in(A,V,T)$ as the input modality and $d_m$ as the vector length. We then put forward the function $\digamma_\Omega^m$ to map the participating user $v.l$ to a point $\varphi_\bullet$ in the $\Omega$.\\
		Consequently, we devise a novel stack-wise approach to convey the routing process. Aiming to build the path for the given user, we can trace a hierarchy from the root toward the user-designated leaf node by obtaining similar child nodes and \textit{pushing} them into stack $S$.\\ 
		Conversely, we process the set of feature vectors ($X^A$, $X^V$, and $X^T$) via a bottom-up approach using the pertinent route from $S$ to predict the sentiments of the input $v$. To this end, we iteratively \textit{pop} a node $\ddot{n}$ from $S$ and append the submodel output from the previous iteration to the \textit{knowledge core} $\ddot{n}^k$, following by executing the node \textit{submodel} $\ddot{n}^m$ on the input feature vectors and the enclosed $\ddot{n}^k$.\\
		Finally, by executing the root submodel in the last iteration, the \textit{predict} module utilizes the submodel output of each utterance as $H\in\mathbb{R}^{\tilde{d}\times g}$ to obtain a set of predictions for each utterance $Y'$. The majority voting on $Y'$ can attain the final cognitive-based sentiment estimation $s$ for $v$\vspace{-5mm}.
		\section{Experiment}
		\label{experiments}
		\vspace{-1mm}
		We conducted extensive experiments on real-world datasets \cite{Majumder2018} to evaluate our proposed model in multimodal sentiment analysis. We developed our algorithms using Tensorflow, Scikit-learn, and Seaborn. Additionally, we performed the experiments on a server with 2.60GHz Intel Core i7-6700HQ CPU, NVidia GeForce 1080 GPU, and 64GB of RAM (available to download \footnote{\url{https://sites.google.com/view/multi-modalsentimentanalysis}})\vspace{-5mm}.		
		\subsection{Configuration}
		\vspace{-1mm}
		\subsubsection{Data}
		\label{dataset}
		\vspace{-1mm}
		For the experiments, we used two publicly available datasets \cite{Majumder2018} from online video-sharing platforms. Table \ref{tab:datasetsinfo} elucidates the corresponding dataset features\vspace{-3mm}.
		\begin{table}[H]
			\def\arraystretch{1.4}
			\begin{tabular}{  c  c  c  c c }
				\hline
				& \# of Users & \# of Videos & Avg Utterance& Density \\ \hline
				CMU-MOSI & 93 & 93 & 24&96.4\% \\ \hline
				MOUD & 55 & 79 & 6&97.6\% \\ \hline
			\end{tabular}
			\vspace{-2mm}
			\caption{Datasets Statistical Description }
			\label{tab:datasetsinfo}
			\vspace{-4mm}
		\end{table}
		\noindent\textbf{CMU-MOSI} \cite{Zadeh2016} includes the multimodal reviews on movie contents in English, where each video comprises multiple segments as utterances. Multiple annotators assign labels between -3 and +3  to utterances. We make the labels binary (positive and negative) as required in experiments.\\
		\textbf{MOUD} \cite{PerezRosas2013} includes video product reviews in the Spanish language. The utterances labeled in neural, positive, and negative classes similarly project into binary values\vspace{-3mm}.
		\subsubsection{Benchmark}
		\label{benchmark}
		\vspace{-1mm}
		Aiming to assess the multimodal sentiment prediction, we tested the statistical hypothesis on the distinctive train and test users for both datasets. We define the statistical parameters as follows: For each utterance $u_j$ composed by a user $c_i$, we determined True Positive when $c_i$ expressed positive sentiment on $u_j$ while the model predicted the same positive sentiment, False Positive when $c_i$ expressed negative sentiment on $u_j$ while the model predicted positive, True Negative when $u_j$ gained negative sentiment while the model predicted the same sentiment, and False Negative when $u_j$ gained positive sentiment while the model predicted the opposite. Finally, we compute the performance metrics of Precision, Recall, and F-measure to compare competitors in the multimodal sentiment analysis\vspace{-3mm}.
		\subsubsection{Baselines}
		\label{baselines}
		\vspace{-1mm}
		This section lists the baselines in multimodal sentiment analysis in both cognitive and non-cognitive categories.\\
		\textbf{CHF} This hierarchical framework\cite{Majumder2018} handles three input modals of acoustics, visual, and textual to compute the sentiments in a hierarchy of uni, bi, and tri-modal layers. \\
		\textbf{CAF} This model\cite{Poria2017a} performs the context-dependent sentiment analysis on utterances using an LSTM network beside an attention-based mechanism to handle fusion.\\
		\textbf{SEP} Similar to \cite{Feuz2015}, this ensemble model utilizes majority voting on base models. The median divides cognitive cues to train submodels on high and low splits.\\
		\textbf{CTP} Our proposed cognitive-based sentiment analysis approach that leverages an attention-based fusion to transfer the agglomerative knowledge within the adaptive tree. As explained in Sec. \ref{sec:Cognitive-space-categorization}, this method constructs the adaptive tree through cluster-based categorization of users.\\
		\textbf{ITP} As elucidated in Sec. \ref{sec:Cognitive-space-categorization}, this approach replicates CTP where it conversely builds the adaptive tree via a theoretical-based categorization\vspace{-3mm}.
		\subsubsection{Data Cognitive Representation}
		\label{sec:data-cognitive-represent}
		One of the most notable points of our framework is that the changes in cognitive features can substantially affect the sentiment analysis results. As elucidated in Sec. \ref{Cognitive_Inference2}, we map each individual to the $\Omega$ space by $\digamma_\Omega^m$ function, resulting in unprecedented cognitive cues.
		The standard density function of underlying distributions can compare the cognitive space of various cues to include OPN, CON, EXT, AGR, and NEU, representing openness, consciousness, extroversion, agreeableness, and neuroticism.
		\begin{figure}[H]
			\vspace{-5mm}
			\centering
			\includegraphics[width=0.8\linewidth]{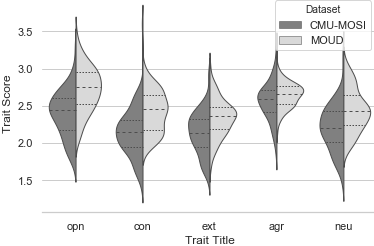}
			\vspace{-3mm}
			\caption{Cognitive cue annotation distribution on MOSI dataset}
			\label{fig:cogcuedistribution}
			\vspace{-3mm}
		\end{figure}
		\noindent As Fig. \ref{fig:cogcuedistribution} depicts, the distribution of cognitive cues, reflected by the users, follows the Gaussian distribution in both datasets, an indication of high impurity that is more centralized for AGR property. Such distribution can constitute a significant contrast between samples that can further facilitate the categorization of users. Compared to MOSI, the cognitive values pertinent to OPN, EXT, and NEU are slightly higher in MOUD. Such traits concerning the sentiment intensity and duration can directly affect emotion manifestation \cite{Verduyn2012}. Nevertheless, we will further need to investigate if categorizing the users can result in meaningful differences between datasets.
		\vspace{-5mm}
		\subsection{Effectiveness}
		\subsubsection{Categorization Performance}
		\label{sec:CategorizationPerformance}
		In this section, we perform an intrinsic evaluation of the clustering-based (CTP) and information-based (ITP) categorization in the adaptive tree (Sec. \ref{sec:Cognitive-space-categorization}).\\
		To examine the cluster quality, we utilize both Silhouette and Davies-Bouldin metrics \cite{Najafipour2020}, reporting the average scores for each level in the hierarchical categorization. Initially, we observe that the more profound the hierarchy grows toward the leaves (level 4), the lower the silhouette score (Fig. \ref{fig:CatPerf}). By contrast, the Davis-Bouldin index, an evaluation based on the cluster centroids, does not pursue a consistent behavior for ITP, reflecting an immense loss in the second level and highlighting a sparse division due to the limitedness of the ITP in generating the split pairs\vspace{-5mm}. 
		\begin{figure}[H]
			\centering
			\begin{subfigure}{0.49\linewidth}
				\centering
				\includegraphics[width=\textwidth]{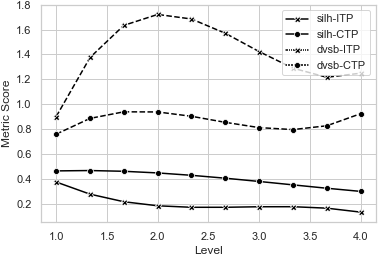}
				\vspace{-5mm}
				\caption{MOSI}
				\label{fig:CatPerf-mosi}
			\end{subfigure}
			\centering
			\begin{subfigure}{0.49\linewidth}
				\centering
				\includegraphics[width=\textwidth]{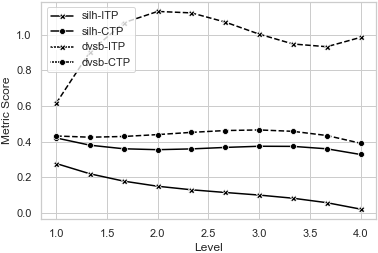}
				\vspace{-5mm}
				\caption{MOUD}
				\label{fig:CatPerf-moud}
			\end{subfigure}
			\hfill
			\vspace{-3mm}
			\caption{Evaluating clusters by Silhouette and Davies Bouldin}
			\vspace{-6mm}
			\label{fig:CatPerf}
		\end{figure}
		\noindent Finally, given higher values for Silhouette (solid lines) and lower numbers for Davis-Bouldin (dashed lines) in Fig. \ref{fig:CatPerf} proves that compared to ITP, the divisions conveyed by the CTP model can better obtain clusters.
		\subsubsection{Impact Of Trilateral Criteria}
		To attain optimal growth in the adaptive tree, as elucidated in Sec. \ref{sec:Cognitive-space-categorization}, we employ the trilateral criteria, proposing three exclusive constraints denoted by $\theta_p$, $\theta_e$, and $\theta_a$. Nevertheless, the more significant the tree growth, the more inevitable efficiency loss will be. Hence, we are to seek thresholds that, on the one hand, can improve the effectiveness of the model, increasing the accuracy, and on the other hand, can preserve the efficiency with a legitimate growth in the adaptive tree.
		\begin{figure}[H]
			\vspace{-3mm}
			\centering
			\begin{subfigure}{0.49\linewidth}
				\centering
				\includegraphics[width=1\textwidth]{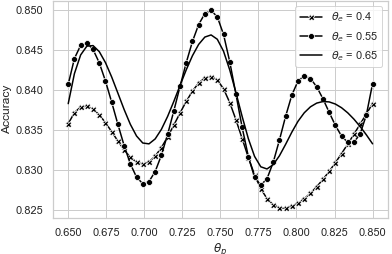}
				\vspace{-6mm}
				\caption{Accuracy}
				\label{fig:CriteriaAccuracyComparison}
			\end{subfigure}
			\centering
			\begin{subfigure}{0.49\linewidth}
				\centering
				\includegraphics[width=1\textwidth]{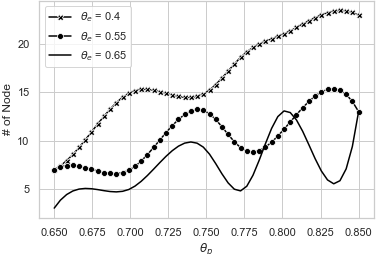}
				\vspace{-6mm}
				\caption{Tree Volume}
				\label{fig:CriteriaVolumeComparison}
			\end{subfigure}
			\hfill
			\vspace{-3mm}
			\caption{Trilateral Thresholds Comparison}
			\label{fig:trilateralcomparison}
			\vspace{-5mm}
		\end{figure}
		\noindent Fig. \ref{fig:trilateralcomparison} collectively examines the impact of recuperation ($\theta_p$) and impurity ($\theta_e$) on the accuracy and the tree volume. The accuracy measures in Fig. \ref{fig:CriteriaAccuracyComparison} show where $\theta_p$ ranges between [0.725, 0.75], $\theta_e$ reflects a positive spike, with better performance for the values of 0.55 and 0.65 where the latter demonstrates more robustness compared to the former.\\
		More specifically, as depicted in Fig. \ref{fig:CriteriaVolumeComparison}, we notice that compared to 0.55, setting $\theta_e$ to 0.65 can better control the size of the tree, nominating $\theta_e=0.65$ as the best fit threshold. In this way, we elegantly sacrifice the negligible change in effectiveness to gain a significant improvement in efficiency.\\
		Fig. \ref{fig:trilateralcomparison3d} further analyzes the selected intervals on the accuracy and tree volume in a more fine-grained view. Finally, if maximizing the performance of the proposed framework in effectiveness and efficiency is of indication, the optimal values of 0.74 and 0.6 for $\theta_p$ and $\theta_e$ can meet the preferences.
		\begin{figure}[H]
			\vspace{-4mm}
			\centering
			\begin{subfigure}{0.49\linewidth}
				\centering
				\includegraphics[width=1\textwidth]{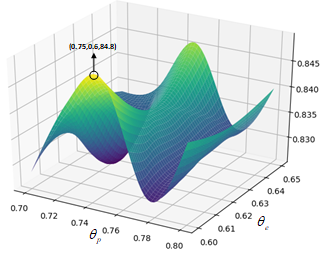}
				\vspace{-5mm}
				\caption{Accuracy}
				\label{fig:CriteriaAccuracyComparison3d}
			\end{subfigure}
			\centering
			\begin{subfigure}{0.49\linewidth}
				\centering
				\includegraphics[width=1\textwidth]{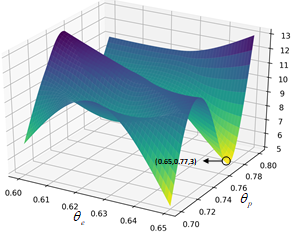}
				\vspace{-5mm}
				\caption{Tree Volume}
				\label{fig:CriteriaVolumeComparison3d}
			\end{subfigure}
			\hfill
			\vspace{-3mm}
			\caption{Trilateral Thresholds Comparison}
			\label{fig:trilateralcomparison3d}
			\vspace{-5mm}
		\end{figure}
		\vspace{-7mm}
		\subsubsection{Impact Of Transferring In The Tree}
		\label{sec:ImpactOfTransfer}
		One notable point about our framework is that it can investigate how the hierarchical transferring of the knowledge from child submodels to the corresponding parent nodes can improve the performance of the sentiment prediction procedure. To empirically track the evidence, we conducted an experiment to examine statistical measures during a trace from leaf nodes toward the root, performed on the cognitive tree.
		More specifically, given the test set associated with node $n$ (Sec. \ref{sec:CognitiveTreeRouting}), we executed sentiment prediction on the submodel $n^m$ using CTP. Finally, we used the average F1-score for all nodes on the same level of the cognitive tree, reported by Fig. \ref{fig:cogtree-leveleval} for MOSI and MOUD datasets of the respective 5 and 3 levels. The higher the level in the hierarchy, the more comprehensive the knowledge core will be from transferring the submodel outputs of the lower levels, resulting in an overall improvement in the tree\vspace{-3mm}.
		\begin{figure}[H]
			\centering
			\begin{subfigure}{0.49\linewidth}
				\centering
				\includegraphics[width=1\textwidth]{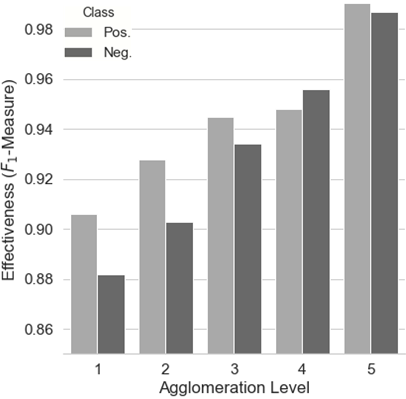}
				\vspace{-5mm}
				\caption{MOSI}
				\label{fig:cogtree-leveleval-mosi}
			\end{subfigure}
			\centering
			\begin{subfigure}{0.49\linewidth}
				\centering
				\includegraphics[width=1\textwidth]{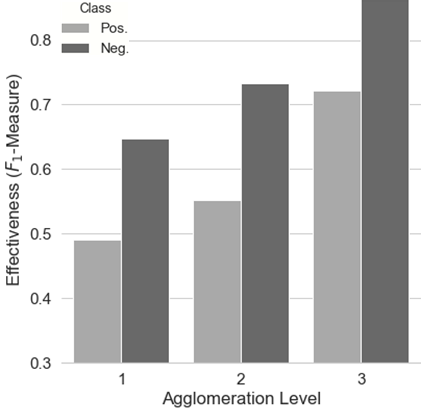}
				\vspace{-5mm}
				\caption{MOUD}
				\label{fig:cogtree-leveleval-moud}
			\end{subfigure}
			\hfill
			\vspace{-3mm}
			\caption{Level-based comparison of F1 in the cognitive tree}
			\label{fig:cogtree-leveleval}
			\vspace{-4mm}
		\end{figure}
		\noindent Furthermore, the expected improvement ratio in the MOUD dataset is more than MOSI, affirming the inference of Sec. \ref{sec:data-cognitive-represent}, where the general cognitive differences imply that the sentiment specifics, including density and duration, vary more for users in the MOUD dataset. Hence, due to a better contrast between users in the MOUD dataset, the cognitive categorization and subsequent transferring modules have substantially improved the initial predictions\vspace{-3mm}.
		\subsubsection{Effectiveness of multimodal sentiment analysis}	
		\label{sec:Effectiveness-of-multimodal-sentiment-analysis}
		Given the benchmark in Sec. \ref{benchmark}, our goal is to examine the competitors' performance (Sec. \ref{baselines}) in multimodal sentiment prediction. Unlike CHF and CAF, the other three, SEP, ITP, and CTP, are cognitive-aware in sentiment inference procedure. Table \ref{tab:baselinecomp} displays the statistics of an average 5-fold test. While $rc^+$ and $rc^-$ report the recall for the respective positive and negative data classes, the pc and $F_1$ represent the precision and F1-score metrics.
		\begin{table}[H]
			\vspace{-3mm}
			\scriptsize
			\centering			
			\caption{Baseline Comparison}
			\vspace{-3mm}
			\def\arraystretch{1.5}
			\begin{tabular}{c|c|c|c|c|c|c|c|c|}
				\cline{2-9}
				& \multicolumn{4}{c|}{MOSI}       & \multicolumn{4}{c|}{MOUD}       \\ \cline{2-9} 
				& pc    & $rc^+$ & $rc^-$ & $F_1$ & pc    & $rc^+$ & $rc^-$ & $F_1$ \\ \hline
				\multicolumn{1}{|c|}{CHF} & 0.865 & 0.862  & 0.869  & 0.865 & 0.694 & \textbf{0.534}  & 0.793  & 0.686 \\ \hline
				\multicolumn{1}{|c|}{CAF} & 0.882 & 0.875  & 0.893  & 0.887 & 0.699 & 0.504  & 0.807  & 0.644 \\ \hline
				\multicolumn{1}{|c|}{SEP} & 0.889 & 0.879  & 0.897  & 0.888 & 0.702 & 0.506  & 0.813  & 0.661 \\ \hline
				\multicolumn{1}{|c|}{ITP} & \textbf{0.895} & \textbf{0.882}  & \textbf{0.9}  & \textbf{0.893} & \textbf{0.717} & 0.505  & \underline{\textbf{0.88}}   & \textbf{0.683} \\ \hline
				\multicolumn{1}{|c|}{CTP} & \underline{\textbf{0.898}} & \underline{\textbf{0.885}}  & \underline{\textbf{0.904}}  & \underline{\textbf{0.896}} & \underline{\textbf{0.727}} & \underline{\textbf{0.542}}  & \textbf{0.844}  & \underline{\textbf{0.695}} \\ \hline
			\end{tabular}
			\vspace{-3mm}
			\label{tab:baselinecomp}
		\end{table}
		\noindent For the MOSI dataset, the cognitive-oriented approaches, including ITP and CTP, significantly improve the performance of non-cognitive methods. Nevertheless, while CTP unanimously gains the best performance in all metrics, the rate of improvement among cognitive-aware models (SEP, ITP, and CTP) is less than 1\%.\\
		More specifically, the MOUD dataset reveals some latent aspects in performance concerning cognitive methods. Given the F1 measure as the main metric to signify the best performance, CTP overpasses other competitors, including ITP. We note that while the number of positive data instances dominates negative data samples, the second-best method, ITP,  plunges up to 4\% less than other rivals (CHF and CTP). Therefore, based on the nature of the dataset, we can rely on 0.695 instead of 0.683 to gain up to 4\% of improvement for positive recall, sacrificing the negligible improvement for the limited negative data samples. It worths mentioning that using ITP in a converse manner can attain up to 4\% of improvement for negative recall when the positive samples may flatten compared to negative data instances. \\
		Since the recommended approaches of ITP and CTP are the primary attempts in utilizing cognitive cues in the multimodal sentiment prediction, we compare the results with non-cognitive baselines. Inherently, we observe the advantage of considering cognitive cues in improving the overall performance, where compared to non-cognitive-based counterparts, the cognitive-based approaches show an average of 2\% improvement in the $F_1$ measure. In a nutshell, the results signify that using cognitive cues can maximize the performance on the negative data class, reflecting an average of 3.6\% increase in the recall.\\
		More specifically, we probe the role of agglomerative ensembles in cognitive-aware methods (ITP, CTP, and SEP). While benefitting the MOUD dataset more, compared to SEP, the $F_1$-measures on ITP and CTP get enhanced significantly in both datasets. The following affirms the empirical results in Sec. \ref{sec:ImpactOfTransfer} that compared to MOSI, the ratio of improvement in each step of agglomeration is considerably higher for MOUD. On the one hand, our proposed CTP framework is capable of creating cohesive splits with higher separation in the adaptive tree (Sec. \ref{sec:CategorizationPerformance}), and on the other hand, by maintaining a relatively higher performance compared to ITP, it is capable of effectively distinguishing the users by pertinent cognitive cues in their brief contents. To conclude, the CTP model overpasses ITP both in the intrinsic and extrinsic analysis, proving that the clustering-based adaptive trees can better track cognitive-aware sentiments\vspace{-4mm}.	
		\subsubsection{Multimodal Sentiment Analysis Complexity}
		\vspace{-1mm}
		In this section, we evaluate the time complexity of our framework from a theoretical computer science perspective. The proposed framework comprises offline and online components, where the former supplies the output of the overall training to the latter, justifying more time consumption by the offline section. Henceforth, we first report the time complexity of the offline components for both CTP and ITP approaches, followed by complexity intuitions about the online section. In retrospect, Eq.\ref{eq:TimeComOverall} verbalizes the overall complexity through aggregating the comprising times in the framework\vspace{-1mm}.
		\begin{equation}
			\small
			\label{eq:TimeComOverall}
			T=T_f+T_c+T^t_b+T^t_m+T_p
			\vspace{-1mm}
		\end{equation}
		Here, $T_f$, $T_c$, $T^t_b$, $T^t_m$, and $T_p$ respectively denote the time complexities of the procedures, including feature extraction, cognitive cue annotation, top-down tree generation, bottom-up model construction, and the final prediction training. Due to the insignificance effect of the feature extraction and cognitive cue annotation modules, the relative time complexities of $T_f$ and $T_c$ fade in contrast with the summation of the other efficiency metrics, Eq. \ref{eq:TimeComOverall}. To this end, we further examine the complexity of $T^t_b$, $T^t_m$, and $T_p$.\\
		\textbf{Tree construction:} $T^t_b$ denotes the complexity for top-down construction of the adaptive tree. Given each node $n$ with $m$ instances, a three-fold set of actions are involved in construction: fragmentation parameter setting $T^t_{par}$, subspace partitioning $T^t_{frag}$, and trilateral termination criteria $T^t_{stop}$.\\ 
		$T^t_{stop}$ is the time requirements to control all the parameters in the worst-case. Given $H$ as the set of splits for the node $n$, we can compute $T^t_{stop}$ by the summation of three complexities, including amplitude $O(|H|\times 1)$, impurity $O(|H|\times 2\times 1)$, and recuperation $O(2\times |H|\times m)$ where the overall complexity for $T^t_{stop}$ can simplify as $O(m)$.\\
		Since the complexity $T^t_{par}$ referring to the parameter setting differs for ITP and CTP models, on the one hand, we select the dimension and the respective cut-point of the node subspace for ITP, and on the other hand, we select the number of clusters in each node for CTP. Correspondingly, the computation complexity of ITP equates to the summation of selection complexities for dimension $O({|C|}^2\times m^2)$ and the cut-point  $O(k\times m^2)$, where $|C|$ and $k$ are respectively the number of dimensions in $\Omega$ space and candidate points in the selection process. Given 5 and 10 as the respective values for $|C|$ and $k$, we can simplify  $T^t_{par}$ to $O(m^2)$ in the ITP model. Similarly, given Eq. \ref{eq:clusternum}, we can obtain the computation complexity for CTP by selecting an optimum number of children as $O(k\times m^2)$, where $k$ counts the constant number of candidates, ending with $O(m^2)$ as the time requirements of $T^t_{par}$.\\
		Finally, we can respectively attain $T^t_{frag}$ for ITP and CTP models using $O(m)$ and $O(k\times m)$, where $k$ denotes the constant number of clusters, resulting in the same complexity of $O(m)$ for both. For ease of computation, the construction complexity in each node can be calculated by the aggregation of the three-fold actions, equating to $O(m+m^2+m)$, where approximately is further simplified to $O(m^2)$. Eq. \ref{eq:TimeComTreeCons} formalizes the complexity of the tree construction ($T^t_b$)\vspace{-1mm}.
		\begin{equation}
			\small
			\label{eq:TimeComTreeCons}
			T^t_b\approx O(m_{01}^2)+O(m_{11}^2+m_{12}^2+\dots)+\dots
			\vspace{-1mm}
		\end{equation}
		Here, we aggregate the computational complexity for the nodes in each level. We further notice that the number of instances in each node is less than the total count in the root $m_{01}$. Where we suppose the criterion $\sum_j m_{ij}\leq m_{01}$ for each level $i$, we can approximately conclude Eq. \ref{eq:TimeComTreeCons} as $T^t_b\approx O(l\times m_{01}^2)$. Accordingly, because of limited capacity for the growth of the adaptive tree, we can disregard $l$, the number of levels, in the final complexity.\\
		\textbf{Tree submodel construction:} Given the training procedure in LSTM models, we assume that the time required for each weight in a time step is $O(1)$. Hence, concerning $W$ as the total number of training parameters, the complexity can be designated as $O(W)$ in each time step. As verbalized in Eq. \ref{eq:TimeComLstmWeights}, $W$ can disregard the biases and include three final parameters of memory cells, input, and output units, respectively denoted by $d_{mc}$, $d_i$, and $d_o$\vspace{-1mm}. 
		\begin{equation}
			\small
			\label{eq:TimeComLstmWeights}
			W=4\times d_{mc}+4\times d_{mc}\times d_i+ d_{mc}\times d_o+3\times d_{mc}
			\vspace{-1mm}
		\end{equation}
		To this end, for a node $n$ with $m$ instances, $O(W\times e\times m\times t)$ can represent the time for training an LSTM, where $e$ is the number of epochs and $t$ is the time steps, equating to the number of utterances within each video (Eq. \ref{eq:MaxUttPerVid}). Both $e$ and $t$ are constant values and can be dismissed \cite{Yuan2020}.\\
		Concerning the same input length, the complexity of the four LSTMs will be identical to $O(W\times m)$. As Eq. \ref{eq:TimeComSubModelCons} elucidates the total summation of the complexity for the proposed adaptive tree $T^t_m$\vspace{-2mm}.
		\begin{equation}
			\small
			\label{eq:TimeComSubModelCons}
			T^t_m\approx O(W\times m_{01})+O(W\times m_{11}+W\times m_{12}+\dots)+\dots
			\vspace{-2mm}
		\end{equation}
		Likewise, we can put a proposition forward to simplify the complexity as $O(l\times W\times m_{01})$.\\
		\textbf{Prediction module:} converting an R-dimensional input to an S-dimensional output demands a time requirement of $O(R\times S)$. Here, the training time $T_p$ will be linear and equal to $O(m_{01})$ with $m_{01}$ as the total number of instances.\\
		Finally, following Eq. \ref{eq:TimeComOverall}, the overall running-time complexity for the offline section will be the same for both CTP and ITP approaches, formalized by Eq. \ref{eq:TimeComOverall2}\vspace{-2mm}:		
		\begin{equation}
			\small
			\label{eq:TimeComOverall2}
			T\approx O(l\times m_{01}^2)+O(l\times W\times m_{01})+O(m_{01})
			\vspace{-2mm}
		\end{equation}
		By employing simplification on the above equation, the computational complexity of the offline section in the framework can be related to $T\approx O(m_{01}^2)$.\\
		Likewise, the time complexity for the online section will be linear with the order of $O(\acute{m})$, where $\acute{m}$ counts the number of input instances. Since the time-sensitive scenarios, including human-computer interaction, can adopt multimodal affective analysis, the running time of the online module deems quite crucial. Therefore, the efficient polynomial approach proposed in this paper can be significantly useful in the practical requirements of real-time frameworks\vspace{-4mm}.
		\section{Conclusion}
		\label{conclusion}
		\vspace{-1mm}
		In this work, we propose a novel unified framework to predict multimodal sentiments by consuming the cognitive cues in user contents. In summary, to leverage individual characteristics, we firstly map each individual to a cognitive space based on exploited latent cues and subsequently perform an agglomerative ensemble to carry out prediction.  Aiming to form the ensemble, we generate the novel adaptive tree by a hierarchical categorization of individuals in the given cognitive space that is further followed by developing transferring-based submodels.  Correspondingly, we employ both clustering and theoretical-based models to classify individuals. Aiming to avoid data sparsity caused by the construction of submodels, we devise a non-trivial deep learning approach based on a dynamic dropout strategy to borrow data from neighboring nodes. Finally, the prediction module leverages the ensemble outcome of the adaptive tree to analyze sentiments, deciding whether the multi-media contents reflect a positive or a negative attitude.\\
		The experimental results on two real-world datasets reveal that leveraging latent cognitive cues can enhance multimodal sentiment analysis, the reason why the proposed solution in this paper outperforms other trending approaches. From another perspective, the proposed agglomerative ensemble can better foster the impact of cognitive cues in sentiment analysis via altering the theoretical-based module by a clustering-based categorization technique. To continue, we can integrate adversarial networks into the adaptive tree to further compensate for data incompleteness and incorporate the personality effects into the fusion-based module to support the attention function. We leave these tasks to tackle in the future\vspace{-4mm}.
		\bibliographystyle{IEEEtran}
		\bibliography{IEEEexample}
		\begin{IEEEbiography}[{\includegraphics[width=1.0in,height=1.25in,clip,keepaspectratio]{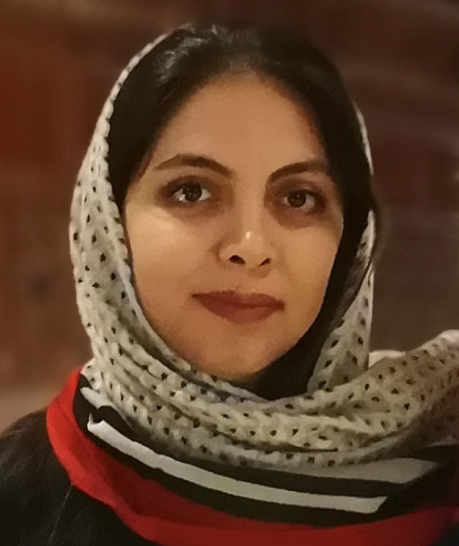}}]{Sana Rahmani} is a researcher in the Computational Cognitive laboratory of Iran University of Science and Technology (IUST). She received M.Sc. from IUST and B.S. from the University of Kurdistan, Iran, in software engineering. Her research interests include multimodal affective analysis, Human-Computer-Interaction, machine learning, and data analysis.
		\end{IEEEbiography}
		\vspace{-14mm}
		\begin{IEEEbiography}[{\includegraphics[width=1.0in,height=1.25in,clip,keepaspectratio]{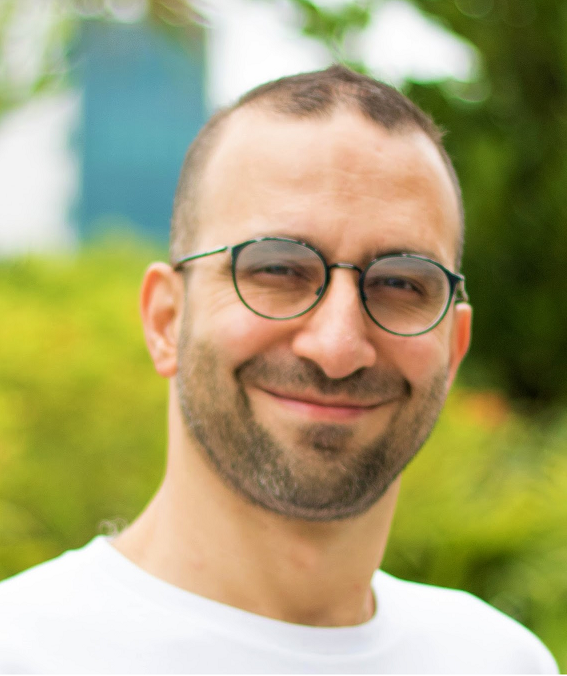}}]{Saeid Hosseini} currently works as an assistant professor at Sohar University. He won the Australian Postgraduate Award and received Ph.D. degree in Computer Science from the University of Queensland, Australia, in 2017. He has also completed two post docs in Singapore and Iran. His research interests include spatiotemporal database, dynamical processes, data and graph mining, big data analytics, recommendation systems, and machine learning.
		\end{IEEEbiography}
		\vspace{-16mm}
		\begin{IEEEbiography}[{\includegraphics[width=1.0in,height=1.25in,clip,keepaspectratio]{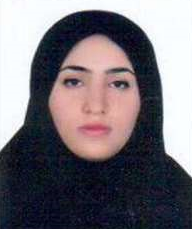}}]{Raziyeh Zall} received the B.S degree from University of Shahid Beheshti of Iran, and M.Sc degree from the Alzahra university of Iran. She is currently working toward the PHD degree in computational cognitive models laboratory at Iran University of Science and Technology. Her research interests include affective and cognitive computing, NLP, and Multi view learning.
		\end{IEEEbiography}
		\vspace{-15mm}
		\begin{IEEEbiography}[{\includegraphics[width=1.0in,height=1.25in,clip,keepaspectratio]{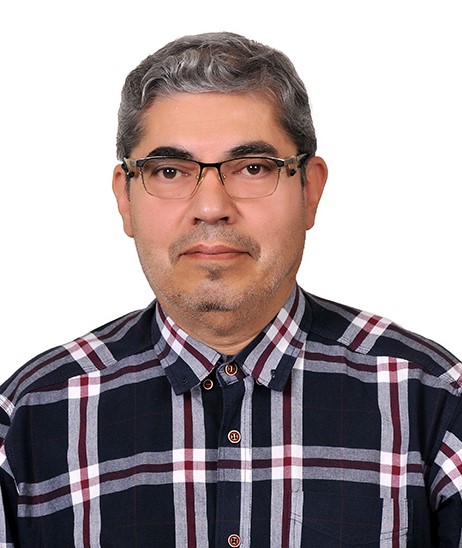}}]{Mohammad Reza Kangavari} received B.Sc. in computer science from the Sharif University of Technology, M.Sc. from Salford, and Ph.D. from the University of Manchester. He is an associate professor at the Iran University of Science and Technology. His research interests include Intelligent Systems, Human-Computer-Interaction, Cognitive Computing, Machine Learning, and Sensor Networks.
		\end{IEEEbiography}
		\vspace{-14mm}
		\begin{IEEEbiography}[{\includegraphics[width=1.0in,height=1.25in,clip,keepaspectratio]{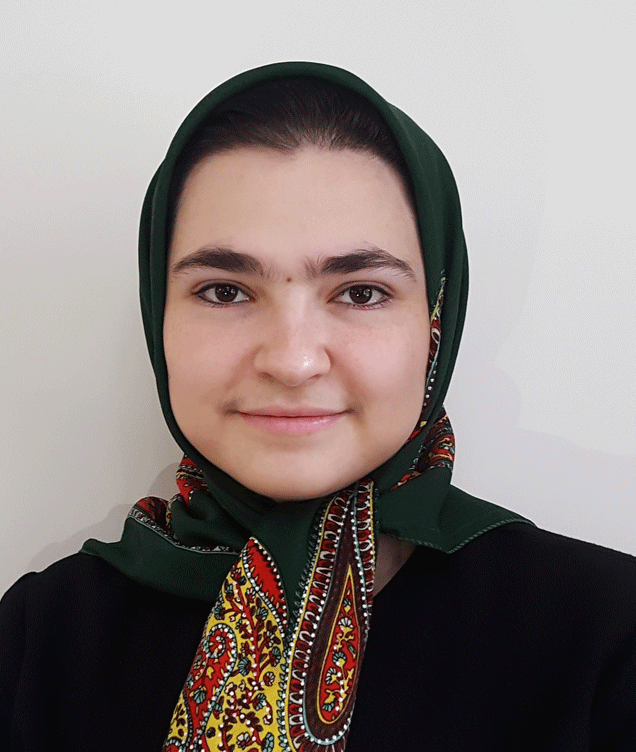}}]{Sara Kamran} is a researcher in the Computational Cognitive laboratory of Iran University of Science and Technology (IUST). She received M.Sc. from IUST and B.S. from the Urmia University of Technology, Iran, in software engineering. Her research interests include affective and cognitive computing, Human-Computer-Interaction, NLP, machine learning, and data analysis.
		\end{IEEEbiography}
		\vspace{-15mm}
		\begin{IEEEbiography}[{\includegraphics[width=1in,height=1.25in,clip,keepaspectratio]{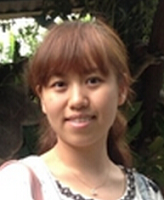}}]{Wen Hua} currently works as a Lecturer at the University of Queensland. She received her doctoral and bachelor degrees in Computer Science from Renmin University. Her current research interests include natural language processing, information extraction and retrieval, text mining, social media analysis, and spatiotemporal data analytics. She has published articles in reputed venues including SIGMOD, TKDE, VLDBJ.
		\end{IEEEbiography}
\end{document}